\documentclass[runningheads]{llncs}
\usepackage[T1]{fontenc}
\usepackage{graphicx}
\usepackage{booktabs}
\usepackage[misc]{ifsym}
\newcommand{\corr}{(\Letter)}
\usepackage{amssymb, amsfonts, mathtools}
\usepackage[dvipsnames, table]{xcolor} 
\usepackage{tabularx, booktabs, multirow, makecell, array}
\usepackage{subcaption}
\usepackage{float, wrapfig, pgfplots}
\usepackage{enumitem, needspace, algorithm, algpseudocode, comment, xspace, textcomp, url}
\usepackage{microtype}
\usepackage[nameinlink]{cleveref} 

\crefname{section}{Sec.}{Secs.}
\crefname{appendix}{App.}{Apps.}
\crefname{table}{Tab.}{Tabs.}
\crefname{app}{App.}{App.}
\Crefname{section}{Sec.}{Secs.}
\Crefname{figure}{Fig.}{Figs.}
\Crefname{appendix}{App.}{Apps.}
\Crefname{definition}{Def.}{Defs.}
\crefname{algorithm}{Alg.}{Algs.}
\Crefname{algorithm}{Alg.}{Algs.}
\crefname{appendix}{App.}{Apps.}
\Crefname{appendix}{App.}{Apps.}

\crefalias{section}{appendix}

\newcommand{\metric}{\text{TCV }}

\newcommand{\mechanismOne}{Path-based Message Passing\xspace}

\newcommand{\mechanismTwo}{Static and Dynamic Propagation Separation\xspace}
\newcommand{\proposed}{\textsc{GLIDE}\xspace}
\newcommand{\fullproposed}{Graph Layer for Inference in Dynamic Environments\xspace}
\setlist[enumerate]{leftmargin=*}

\newcommand{\TCVSolar}{\ensuremath{0.323}}

\newcommand{\TCVElectricity}{\ensuremath{0.194}}
\newcommand{\TCVExchange}{\ensuremath{0.007}}
\newcommand{\TCVGermany}{\ensuremath{0.997}}
\newcommand{\TCVFrance}{\ensuremath{0.969}}
\newcommand{\TCVETThOne}{\ensuremath{0.679}}

\newcommand{\TCVStatic}{\ensuremath{0.296}}
\newcommand{\TCVEasy}{\ensuremath{0.296}}

\newcommand{\TCVHard}{\ensuremath{0.874}}
\newcommand{\TCVDynamic}{\ensuremath{0.874}}

\newcommand{\modify}[1]{{\textcolor{black}{#1}}}

\newcommand{\fixme}[1]{\textcolor{black}{#1}}

\newcommand{\Prealmax}{85.7}

\graphicspath{{imgs/}}

\begin{document}
\setlength{\tabcolsep}{1pt}
\title{When GNNs Fail: Quantifying and Overcoming Temporal Correlation Volatility in Time Series}

\titlerunning{Overcoming Temporal Correlation Volatility in GNNs}

\author{Chen Shao\inst{1} \and
Yue Wang\inst{1} \and
Zhenyi Zhu\inst{2} \and
Zhanbo Huang\inst{1} \and
Tobias K\"afer\inst{1} \and
Zonghan Wu\inst{3} \and
Danai Koutra\inst{4}\corr}

\authorrunning{C. Shao et al.}

\toctitle{When GNNs Fail: Quantifying and Overcoming Temporal Correlation Volatility in Time Series}
\tocauthor{Chen Shao, Yue Wang, Zhenyi Zhu, Zhanbo Huang, Tobias K\"afer, Zonghan Wu, Danai Koutra}

\institute{Karlsruhe Institute of Technology, Germany\\
\email{cc7738@kit.edu, urryl@student.kit.edu}\\
\email{zhanbo.huang@rwth-aachen.de, tobias.kaefer@kit.edu}
\and
The Hong Kong University of Science and Technology, Hong Kong SAR\\
\email{zzhubh@connect.ust.hk}
\and
East China Normal University, China\\
\email{zhwu@sem.ecnu.edu.cn}
\and
University of Michigan, USA\\
\email{dkoutra@umich.edu}}


\maketitle            

\begin{abstract}
Modeling multivariate time series by representing them as graphs, where individual series act as nodes and pairwise temporal correlations serve as edges, has gained significant traction. Recent advances in Graph Neural Networks (GNNs) have demonstrated strong performance by assuming a static graph topology and aggregating information from neighboring series. In this work, we investigate the representational power of GNNs for forecasting under both static and dynamic settings (i.e., when pairwise correlations evolve drastically over time) and identify critical limitations in current architectures. To formalize this, we first propose Temporal Correlation Volatility (TCV), a model-agnostic metric designed to quantify the distributional evolution of these latent structures. We establish a clear connection between TCV and performance degradation, demonstrating that many popular models, including Transformers, generalize poorly in high-TCV settings and are often outperformed by simple structure-agnostic baselines. To address these limitations, we propose \fullproposed\ (\proposed), a novel GNN layer enhanced by two theoretically grounded design mechanisms: (D1) \mechanismOne, which captures path-based neighborhoods and (D2) \mechanismTwo, which identifies optimal dynamics via local static approximation. These components significantly improve learning under dynamic topology while preserving robustness in static scenarios. Extensive experiments on synthetic and real-world benchmarks show that \proposed improves average performance by up to 45.6\% across static and dynamic settings, with the largest gain reaching 85.7\%. The source code is available at \url{https://github.com/ChenS676/GLIDE}. 
\end{abstract}

\begin{figure}[b!]
    \centering
    \includegraphics[width=0.88\linewidth]{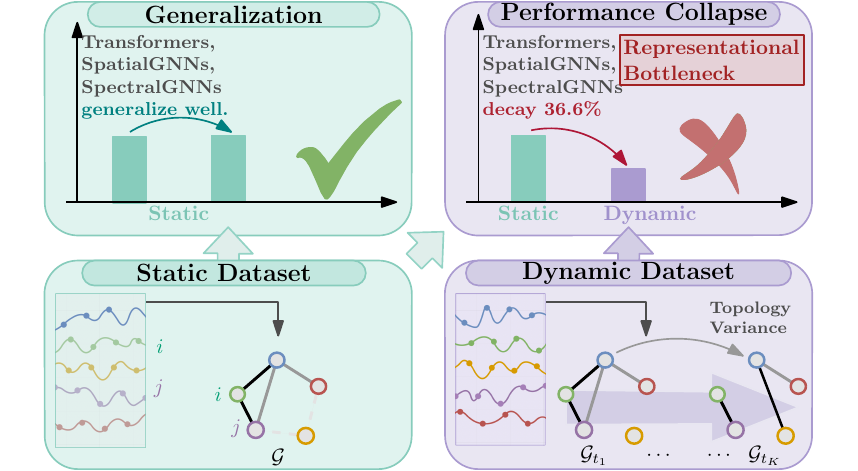}
    \caption{Core motivation of our study: GNN/Transformers fail when the underlying graph structure undergoes drastic evolution (\Cref{sec:motivating example}).}\label{fig:motivate}
\end{figure}%

\section{Introduction}
Multivariate time series forecasting (MTSF) is a pivotal challenge in domains such as energy management, weather prediction and financial modeling \cite{shao2025data,sezer2020financial}. Recently, Graph Neural Networks (GNNs) have transformed this field by treating time series as graph-structured processes, where variables are modeled as nodes in a relational network \cite{wu2020connecting}. Within this paradigm, identifying the graph topology, both its connectivity and relational strength, is a critical prerequisite. This topology serves as the structural foundation for spatial message passing, dictating how information is aggregated across variables. An inaccurate or static topology can lead to the "over-smoothing" of distinct signals or the aggregation of spurious correlations, fundamentally degrading the model's ability to resolve complex dependencies.

Researchers often model the graph using a predefined adjacency matrix derived from explicit structures, infer a latent topology via global similarity metrics computed over the entire historical sequence (e.g., cosine similarity, Gaussian kernels~\cite{li2018diffusionconvolutionalrecurrentneural}, or Pearson correlation~\cite{cao2021spectraltemporalgraphneural,shang2021discrete,liu2022multivariate}), or use network-discovery methods~\cite{SafaviSK17,BrugereGB18}.
Global estimation fundamentally assumes a static graph topology, wherein the latent relational structure and pairwise dependencies are treated as approximately invariant over the entire temporal horizon. This approach has become standard practice in prominent modern architectures~\cite{wu2020connecting,liu2022multivariate,yi2023fouriergnn}. From a theoretical perspective, this global estimation amounts to a low-rank approximation of what is inherently a time-varying manifold. 

The limitations of this static paradigm become evident in highly dynamic environments where inter-variable relationships undergo abrupt structural breaks \cite{liu2022multivariate,shao2025data}. For instance, in power grid management, the dependency between two energy nodes may shift abruptly due to localized grid failures or fluctuating demand cycles \cite{nauck2022dynamic}; similarly, macroeconomic shocks can trigger structural breaks that fundamentally alter dependencies among financial assets, rendering global correlation estimates obsolete \cite{sezer2020financial}. In such settings, existing approaches fail to explicitly encode temporal variability in inter-series dependencies, leading to sub-optimal forecasting performance. We demonstrate in \Cref{tab:synthetic-results} that, in the worst-case scenario, performance degradation can reach up to 36.6\%, including popular methods such as Transformers \cite{wu2021autoformer}. 
Despite these advances, existing methodologies suffer from three fundamental limitations. First, contemporary temporal polynomial approaches aim to model dynamic correlations through time-varying coefficients \cite{liu2022multivariate,yue2025olinearlinearmodeltime}. However, these adaptive parameters frequently fail to synchronize with distribution shifts or macro-structural variances in the underlying graph topology \cite{campbell2023dyndepnetlearningtimevaryingdependency} and often exhibit suboptimal robustness in real-world benchmarks characterized by evolving adjacency patterns \cite{shao2025data}. Second, the prevailing extract-then-forecast paradigm creates a representational decoupling: by treating graph estimation and time-series prediction as disjoint optimization objectives, these models fail to propagate structural uncertainty through the message-passing layers \cite{shang2021discrete}. Finally, the absence of a quantitative diagnostic framework to measure topology perturbation means that the boundary conditions under which these models fail remain poorly understood, necessitating a more rigorous, metric-driven approach to structural dynamics.  

In contrast to the common assumption of static conditions, this work contributes by specifically investigating the representational power of graph-based approaches in non-static settings. This paper aims to address these limitations by systematically analyzing the behavior of graph patterns within dynamic MTSF frameworks. We begin by providing a theoretical characterization of correlation matrices under non-static conditions. Our findings reveal that direct correlations in dynamic graph features no longer adhere to static assumptions; instead, in dynamic scenarios, these direct dependencies become highly unstable. To quantify the inherent instability of these evolving structures, we introduce \textit{Temporal Correlation Volatility} (TCV), a metric designed to measure the fluctuation intensity of pairwise dependencies between variables. In this work, we use it as a diagnostic tool: through an extensive empirical study involving {18 baseline models}, we demonstrate that the performance degradation of existing forecasting frameworks is highly correlated with the structural volatility(quantified by TCV) of the underlying graph. Building on these insights, we propose a novel \proposed layer that explicitly accounts for temporal variability in graph structures through two key design components: (D1)~\mechanismOne, which constructs path-based interaction features to induce more robust graph representations; and (D2)~\mechanismTwo, which addresses temporal shifts by employing an orthogonal basis decomposition to decouple correlations into persistent static and transient dynamic components. In summary, our contributions are as follows:
\begin{itemize}
    \item \textbf{Quantitative Characterization of Topology Variance.} We introduce \emph{Temporal Correlation Volatility} (TCV), a principled metric for quantifying non-static topology shifts in multivariate time series. We provide a rigorous theoretical foundation for TCV and empirically establish its correlation with forecasting performance degradation. This allows TCV to serve as a diagnostic framework for assessing GNN robustness under dynamic structural breaks.
    \item \textbf{The \proposed Architecture for Dynamic Inductive Bias.} We develop \proposed, featuring a novel GNN layer that incorporates \emph{higher-order interactions} and \emph{time-evolving dynamics decomposition}. \fixme{We establish a consistency-style guarantee for a kernel-weighted estimator of the time-varying topology, which motivates D2's decomposition into persistent global structure and transient local fluctuations.} This ensures the model endogenously handles structural uncertainty.
    \item \textbf{Comprehensive Benchmarking.} {We conduct extensive experiments against {18 baselines across eight benchmarks}, including both static and dynamic topology regimes. Results show that \proposed consistently outperforms state-of-the-art methods in both static and dynamic settings, yielding average improvements of $45.6\%$ and reaching a maximum gain of $\Prealmax\%$ on the Exchange Rate benchmark at the 12-step forecasting horizon.}
\end{itemize}
\section{Preliminaries}
\label{sec:preliminary}
In this section, we present problem formulation and key definitions that we use in our work. Let $\mathbf{x}_{t} \in \mathbb{R}^N$ be a multivariate observation at time $t$, with $x_{t}[i]$ as the $i$-th variable. We define the look-back window of length $T$ as $\mathbf{X} = [\mathbf{x}_{t-T+1}, \dots, \mathbf{x}_{t}] \in \mathbb{R}^{N \times T}$. We aim to learn a mapping $f: \mathbf{X} \to \mathbf{Y}$, where target $\mathbf{Y}$ is either a single-step vector $\mathbf{x}_{t+1}$ or multi-horizon sequence $\{\mathbf{x}_{t+1}, \dots, \mathbf{x}_{t+P}\}$. From a graph-based perspective, $\mathbf{X}$ is a temporal graph sequence, where variables are represented as nodes and cross-variable dependencies are encoded in an adjacency matrix. Formally, the graph series is $\mathcal{G}(t) = (\mathcal{V}, \mathcal{E}_t, \mathbf{X}_t, \mathbf{A}_t)$, where $\mathcal{V}$ is the vertex set, $\mathbf{X}_t \in \mathbb{R}^{N \times B}$ are node features from a sliding window of length $B$ and $\mathbf{A}_t \in \mathbb{R}^{N \times N}$ is the sparsified weighted adjacency structure.
\begin{definition}[Topology-based Graph]\label{def:adj}The adjacency matrix $\mathbf{A} \in \mathbb{R}^{N \times N}$ is constructed as a $k$-nearest neighbor ($k$-NN) graph, where each entry $\mathbf{A}_{ij}$ is defined as $a_{ij}$ if $v_j \in \mathcal{N}^k(v_i)$ and $0$ otherwise. $\mathcal{N}^k(v_i)$ denotes the set of $k$ nodes most similar to $v_i$ according to a chosen distance metric.\end{definition}
In the absence of a predefined topology, the adjacency structure is inferred from the empirical properties of the observed signals.
\begin{definition}[Similarity-based Graph]\label{def:sim_graph}
Given a data matrix $\mathbf{X} \in \mathbb{R}^{N \times T}$, let $\tilde{\mathbf{X}}$ denote row-wise standardized node signals. The similarity-based adjacency matrix $\hat{\mathbf{A}} \in \mathbb{R}^{N \times N}$ is defined via empirical correlation: 
$\hat{a}_{ij} = \frac{1}{T}\textstyle\sum_{t=1}^T \tilde{x}_{it}\tilde{x}_{jt}$, implying $\hat{\mathbf{A}} = \frac{1}{T}\tilde{\mathbf{X}}\tilde{\mathbf{X}}^{\top}$
\noindent Here, $\hat{a}_{ij} \in [-1,1]$ is the Pearson correlation coefficient, measuring the linear dependence between node signals $v_i$ and $v_j$.
\end{definition}
\textbf{Remark.} While alternative similarity measures such as Dynamic Time Warping (DTW) or non-linear kernels exist \cite{wu2020connecting}, this work focuses on linear correlation properties. This choice is motivated by their widespread application \cite{wu2020connecting}, high computational efficiency and scalability for high-dimensional graphs \cite{8347162}. In this work, we identify the neglect of these linear time-varying dependencies as a fundamental limitation of existing dominant methods.
\subsection{Measuring Structural Dynamics: Temporal Correlation Volatility}
\label{subsec:TCV}
In real-world scenarios, the dependencies between nodes are time-varying. Motivated by work on graph similarity and change detection~\cite{KoutraVF13,BerlingerioKEF13},  we introduce Temporal Correlation Volatility (TCV), a model-agnostic measure that quantifies the intensity of temporal shifts and characterizes the distributional evolution of spatial correlation structures over time.
\begin{definition}[Temporal Correlation Volatility]
\label{def:TCV}
{For each timestamp $t$, we first construct a similarity graph $\hat{\mathbf{A}}_t$ from a sliding window of observations $\mathbf{X}_t \in \mathbb{R}^{N\times B}$ according to \Cref{def:sim_graph}. The resulting sequence $\{\hat{\mathbf{A}}_1,\ldots,\hat{\mathbf{A}}_T\}$ captures the time-varying dependency structure of the multivariate series.} Given a multivariate time series $\mathbf{X}_t \in \mathbb{R}^{N \times T}$ and an inferred similarity graph $\hat{\mathbf{A}}_t \in [-1, 1]^{N \times N}$ at time $t$, the \metric is defined as the average rate of topological change quantified by the Frobenius norm $\|\cdot\|_F$:
\begin{equation}
    \mathrm{TCV} = \frac{1}{2(T-1)N} \sum_{t=2}^{T} \|\hat{\mathbf{A}}_{t} - \hat{\mathbf{A}}_{t-1}\|_F
    = \frac{1}{2(T-1)N} \sum_{t=2}^{T} \sqrt{\sum_{i,j} (\hat{a}_{ij}^{(t)} - \hat{a}_{ij}^{(t-1)})^2},
    \label{eq:TCV}
\end{equation}
Empirical similarity-based adjacency matrix $\hat{\mathbf{A}}_t$ is derived from sliding windows of the input signals (\Cref{def:sim_graph}). {\metric quantifies the temporal volatility of the graph series by averaging the magnitude of change between consecutive graph states.} {The instantaneous topological transition is measured as the Frobenius distance between consecutive similarity matrices, $\|\hat{\mathbf{A}}_{t+1}-\hat{\mathbf{A}}_t\|_F$.} Since each $\hat{\mathbf{A}}_t$ is normalized per variable to the range $[-1, 1]$, the inclusion of the $2N$ factor in the denominator ensures that {$0 \le \mathrm{TCV} < 1$}. \end{definition}
\textbf{\metric in Practice.} The \metric effectively distinguishes between structurally static regimes ($\mathrm{TCV}\to0$), where the graph backbone remains stable; for example, the \textit{Exchange Rate} dataset exhibits a very low value ($\mathrm{TCV} \approx 0.007$), indicating a highly stable long-term topology in the foreign exchange market. In contrast, high $\mathrm{TCV}\to 1$ indicates that consecutive graph states change substantially over time; that is, the \textit{Solar} and \textit{German Energy} datasets~\cite{shao2025data} attain substantially higher values ($\mathrm{TCV} > 0.5$), as they capture major disruptions such as the Crimea crisis and the Russia--Ukraine war. {Higher \metric values correspond to larger deviations between consecutive graph states and therefore stronger violations of the implicit stationarity assumptions adopted by many existing graph-based forecasting models. As we demonstrate in the following motivating example, forecasting performance deteriorates significantly as topological volatility increases.}

\subsection{Connecting Topology Dynamics and Performance Variation}\label{sec:motivating example}
We then investigate the performance discrepancies across four dominant representative categories under both static and dynamic structural regimes and delineate the boundary conditions under which current dominant models fail. Specifically, we demonstrate that prevailing methods struggle to generalize to highly non-static sequences characterized by rapidly evolving, time-varying topologies. 
\paragraph{Synthetic Data.} 
Following \cite{liu2022multivariate}, we generate a multivariate time series (MTS) of length $T$ with $N$ variables using a set of $N_w$ graph transition matrices $\{\mathbf{A}_1,\dots,\mathbf{A}_{N_w}\}$. The active graph switches sequentially every $T_s=T_{\mathcal G}/N_w$ timesteps. While $\mathbf{A}_i$ is active, observations are generated through random walks governed by $\mathbf{A}_i$. Thus, $N_w$ determines the number of graph states and $T_s$ the switching frequency. The static case uses a single graph ($N_w=1$), whereas the dynamic case employs multiple graph states ($N_w>1$). Full details are given in App.\,\Cref{algorithm:syn}.

As shown in \Cref{tab:synthetic-results}, spatial, spectral and Transformer architectures remain competitive in static regimes but suffer significant degradation on dynamic datasets. Notably, Transformer models experience the most acute collapse (up to 36.6\%), suggesting that implicit graph induction via self-attention is highly susceptible to structural instability in high-dynamic scenarios. Explicit graph models exhibit error increases between 15.9\% and 27.2\%. We show later that, \Cref{sec:eval} also confirms these gaps persist across real-world benchmarks. This decay stems from representational limitations rather than dataset statistics, underscored by the structure baseline; though globally inferior, it remains robust to topological fluctuations because it lacks graph-dependent inductive biases.
 
This motivating example elucidates two critical limitations of the current landscape: (i) existing graph-centric architectures fail to maintain structural robustness on datasets characterized by significant non-static, time-evolving topologies; and (ii) this robustness depends on whether the architecture carries an inductive bias for evolving topologies, rather than implicitly assuming a static graph. This finding provides rigorous empirical evidence of the representational bottlenecks in existing methodologies. It highlights a critical, yet largely unexplored, frontier in time-series forecasting: the modeling of rapid topological shifts, which we systematically investigate in this work.
\begin{table}[t]
    \centering
    \renewcommand{\arraystretch}{0.85} 
    \caption{Comparison of MAE performance degradation ($\Delta\text{MAE} = \text{MAE}_{\text{d}} - \text{MAE}_{\text{s}}$) for seven strong baselines under static (Syn-S) and dynamic (Syn-D) conditions. 
    {Note that lower MAE is better}. Higher TCV leads to consistent degradation across graph and Transformer baselines, while structure-free MLP degrades less.}
    \label{tab:synthetic-results}
    \footnotesize
    \begin{tabular}{llccc}
        \toprule
        \multirow{2}{*}{\textbf{Family}} 
        & \textbf{Method}
        & \textbf{Syn-S} 
        & \textbf{Syn-D} 
        & \textbf{Relative} \\
        & TCV & \TCVStatic 
        & \TCVDynamic 
        & \textbf{$\Delta$MAE (\%)} \\
        \midrule
        \multirow{2}{*}{\textbf{SpatialGNN}}
        & MTGNN \cite{wu2020connecting}        
        & $0.427$ & $0.508$ & $-19.0\% \downarrow$ \\
        & TPGNN \cite{liu2022multivariate}      
        & $0.456$ & $0.578$ & $-26.8\% \downarrow$ \\
        \midrule
        \multirow{2}{*}{\textbf{SpectralGNN}}
        & StemGNN \cite{cao2021spectraltemporalgraphneural}    
        & $0.497$ & $0.576$ & $-15.9\% \downarrow$ \\
        & FourGNN \cite{yi2023fouriergnn}
        & $0.651$ & $0.828$ & $-27.2\% \downarrow$ \\
        \midrule
        \multirow{3}{*}{\textbf{Transformer}}
        & Autoformer \cite{wu2021autoformer} 
        & $0.650$ & $0.888$ & $-36.6\% \downarrow$ \\
        & {Reformer \cite{Kitaev2020ReformerTE}} 
        & $0.647$ & $0.855$ & $-32.1\% \downarrow$ \\
        \midrule
        \textbf{Structure-free}
        & MLP          
        & $0.809$ & $0.927$ & $-14.6\% \downarrow$ \\
        \bottomrule
    \end{tabular}
\end{table}
\section{\proposed: \fullproposed}\label{sec:model}
Motivated by these limitations, we develop a novel layer centered on two architectural designs: (D1) \mechanismOne and (D2) \mechanismTwo. We introduce and justify this layer theoretically and then introduce our full model \proposed in \Cref{sec:glide}.
\subsection{(D1) \mechanismOne}
\subsubsection{Intuition.} When cross-variable correlations exhibit high volatility, direct edges in the adjacency matrix $\hat{\mathbf{A}}$ become susceptible to stochastic oscillations and transient noise. In contrast, indirect dependencies, particularly those derived from higher-order topological neighborhoods, often emerge as more stable indicators of underlying dynamics. For instance, during the 2021--2022 European energy crisis, as observed in the France energy dataset \cite{shao2025data}, coal and various renewable energy sources maintained a systematic, sustained coupling mediated through the fossil fuel market (e.g., Coal$\to$Carbon$\to$Renewables) rather than through direct correlations. This suggests that higher-order paths can capture latent structural relationships that remain robust despite the instability of local edges.

In \proposed, time series feature aggregation is centered on a path-based neighborhood. Mechanically,  the node representations $\mathbf{H}$ are updated via a linear weighted combination of the path-based correlation based on the similarity graph $\hat{\mathbf{A}}$. We derive one special case of the second-order neighborhood to clarify the idea. $R_{ij}^{(2)}$ represents the sum of all 2-hop paths between node $i$ and node $j$ defined as $R_{ij}^{(2)} = \sum_{m=1}^{n} \hat{A}_{im} \hat{A}_{mj} = [\hat{A}^2]_{ij}$. By generalizing this reachability to a $K$-th order expansion, we define a composite polynomial reachability element $R_{ij}^{(\text{w})}$, which represents the total weighted influence of node $j$ on node $i$ across various path lengths. With a given node temporal features $\mathbf{H}_i$ of node $i$ (from the temporal convolution layer introduced later), the local update rule for the feature vector of a specific node $i$ is formulated as:
\[\mathbf{H}_i = \sigma \left( \sum_{j} R_{ij}^{(\text{w})} \mathbf{H}_j \mathbf{W} \right),\quad R_{ij}^{(\text{w})} = \left[ \sum_{k=0}^{K} \mathbf{w}^k \hat{\mathbf{A}}^k \right]_{ij}\]
This enables the model to bypass transient first-order noise and learn from the stable, latent connectivity inherent in complex multivariate time series.
\begin{figure}[h]
    \centering
    \includegraphics[width=\linewidth]{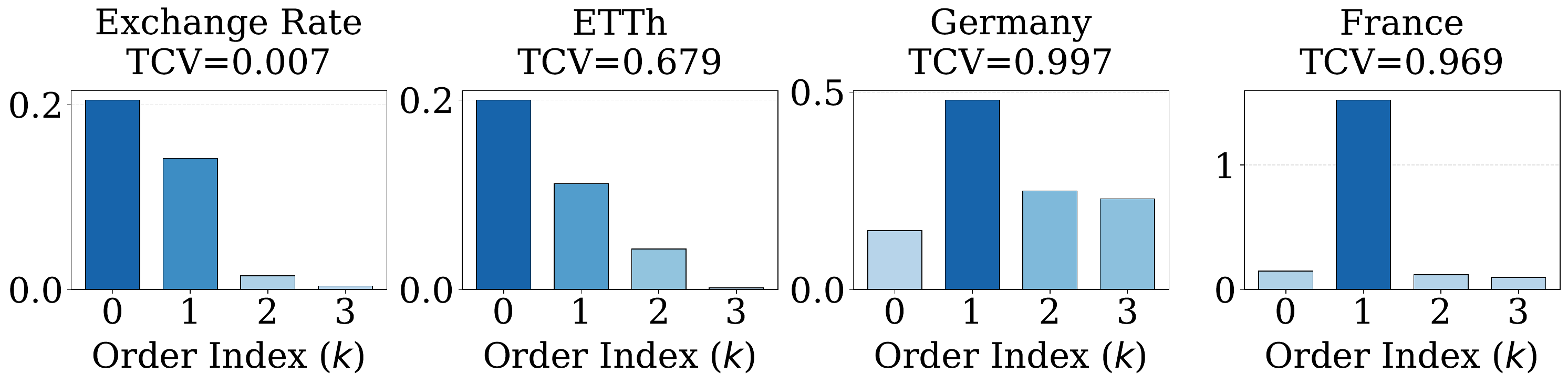}
    \caption{Distribution of learned coefficient magnitudes $\|w_k\|_2$ in Design 1. Static data (Exchange) exhibit a decay toward lower-order terms, whereas datasets with dynamic topology (Germany, France) show a marked shift in weight toward higher-order dependencies.}
    \label{fig:alpha_distribution}
\end{figure}
\subsection{(D2) \mechanismTwo} 
While real-world similarity graphs are typically assumed to be static, with topology drawn independently and identically distributed (i.i.d.) over time, this assumption is often violated in dynamic environments. Here we derive a basic form for the underlying graph structure in a dynamic setting. Without loss of generality, we consider a sequence of graph topologies $\{\mathbf{A}_t\}$ and their estimators $\{\widehat{\mathbf{A}}_t\}$. We now present the following theorem.

\begin{theorem}[Dynamic Topology Identification]\label{thm:d2-gnn}
Consider a graph topology evolving as $\mathbf{A}_t = \mathbf{A}_{t-1} + \mathbf{Z}_t$, with initial state $\mathbf{A}_0 \sim \mathcal{N}(0,\boldsymbol{\Sigma}(0))$ and temporal perturbation $\mathbf{Z}_t \sim \mathcal{N}(0,\boldsymbol{\Sigma}(t))$. Assume that the temporal perturbations $\{\mathbf{Z}_t\}_{t=1}^{T}$ are mutually independent across time and that $\boldsymbol{\Sigma}(t)$ varies smoothly with $t$. Under mild sparsity and regularity conditions, the time-varying topology at time $t$ can be identified via {an $\ell_1$-regularized inverse-covariance estimator (Graph-LASSO \cite{kolar2010estimating})} from the kernel-weighted similarity
$\widehat{\mathbf{A}}_t = \frac{\sum_s w_{st}\,\widetilde{\mathbf{X}}_s \widetilde{\mathbf{X}}_s^\top}{\sum_s w_{st}} = \frac{\sum_s w_{st}\,\widetilde{\mathbf{A}}_s}{\sum_s w_{st}},$
where $w_{st} = K(s-t)$ is a kernel weight measuring temporal proximity between time steps $s$ and $t$.
\end{theorem}
\subsubsection{Justification.} \fixme{\Cref{thm:d2-gnn} shows that the time-varying topology can be recovered through a static-graph estimation problem by aggregating temporally neighbouring observations through a linear weighting.} \fixme{The second design encodes \emph{static graph} (stable, time-invariant) separately from \emph{dynamic graph} (time-varying shocks) during aggregation.} Formally, the representation at layer $\ell$ is updated as:
\[\mathbf{H} = \sigma \Big( \big( \hat{\mathbf{A}}_{\text{s}}  \mathbf{W}_{\text{s}} + \hat{\mathbf{A}}_{\text{d}}\mathbf{W}_{\text{d}} \big)\mathbf{H}  \Big)\]
where $\hat{\mathbf{A}}_{\text{s}}$ represents time-invariant topology, inferred either from time series or predefined. $\hat{\mathbf{A}}_{\text{d}}$ encodes time-dependent topology variance.
\fixme{We instantiate \Cref{thm:d2-gnn} for the dynamic branch with a box kernel of width $B$, which yields two natural choices. The \emph{raw-signal} variant computes a short-window local similarity,
\begin{equation}
\label{eq:corr_graph}
    \hat{\mathbf{A}}_{\text{d}}^{(t)} = \mathrm{ReLU}\!\big(\tfrac{1}{B}\textstyle\sum\nolimits_{k=t-B/2}^{t+B/2} \tilde{\mathbf{x}}_k \tilde{\mathbf{x}}_k^{\top} \mathbf{M}_{\text{d}}\big),
\end{equation}
a direct realisation of \Cref{thm:d2-gnn} with $K(s{-}t)=\tfrac{1}{B}\,\mathbf{1}\{|s-t|\le B/2\}$. The \emph{gradient} variant, defined in Eq.~\eqref{eq:grad_graph} below, replaces $\tilde{\mathbf{x}}_k$ with the first-order temporal difference $\nabla \tilde{\mathbf{x}}_k = \tilde{\mathbf{x}}_k - \tilde{\mathbf{x}}_{k-1}$ so the dynamic branch isolates transient shocks while the static branch absorbs the slow-varying correlation. In both variants, the learned $\mathbf{M}_{\text{d}}$ together with ReLU echo the non-negativity and sparsity that the Graph-LASSO penalty enforces in the theorem. We adopt the gradient variant in our main experiments and report an ablation against the raw-signal variant in \Cref{app:details of ablation study}.}
\begin{equation}
\label{eq:grad_graph}
\hat{\mathbf{A}}_{\text{d}}^{(t)} = \text{ReLU}(\frac{1}{B} \sum\nolimits_{k=t-B/2}^{t+B/2} (\nabla \tilde{\mathbf{x}}_{k})(\nabla \tilde{\mathbf{x}}_{k})^{\top} \mathbf{M}_{\text{d}}),
\end{equation}
{where $\mathbf{M}_{\text{d}} \in \mathbb{R}^{N \times N}$ is a learnable mixing matrix that re-scales the empirical local outer product before the ReLU, parameterizing the dynamic adjacency. \fixme{Focusing on the gradient, $\hat{\mathbf{A}}_{\text{d}}^{(t)}$ explicitly encodes the synchronization of rapid shocks and directional shifts across variables, providing a contrast to the stable correlations captured by the static backbone.}
\begin{figure}[t]
    \centering
    \includegraphics[width=0.95\linewidth]{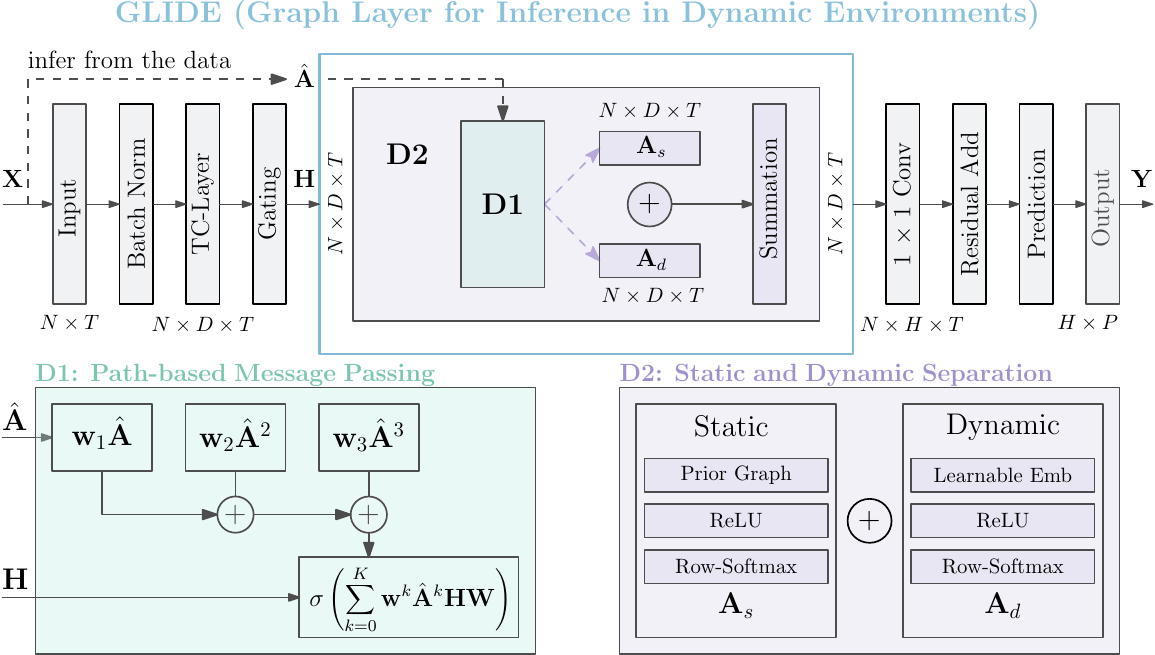}
    \caption{\modify{The framework of \proposed. It consists of K TC-layers, 1-GLIDE layer and one prediction layer. The inputs are first transformed by a normalization layer and then passed to the TC-Layer followed by the GLIDE. Each layer has residual connections and is skip-connected to the output layer.}}
    \label{fig:overview}
\end{figure}
\begin{figure}[t]
    \centering
    \includegraphics[width=0.7\linewidth]{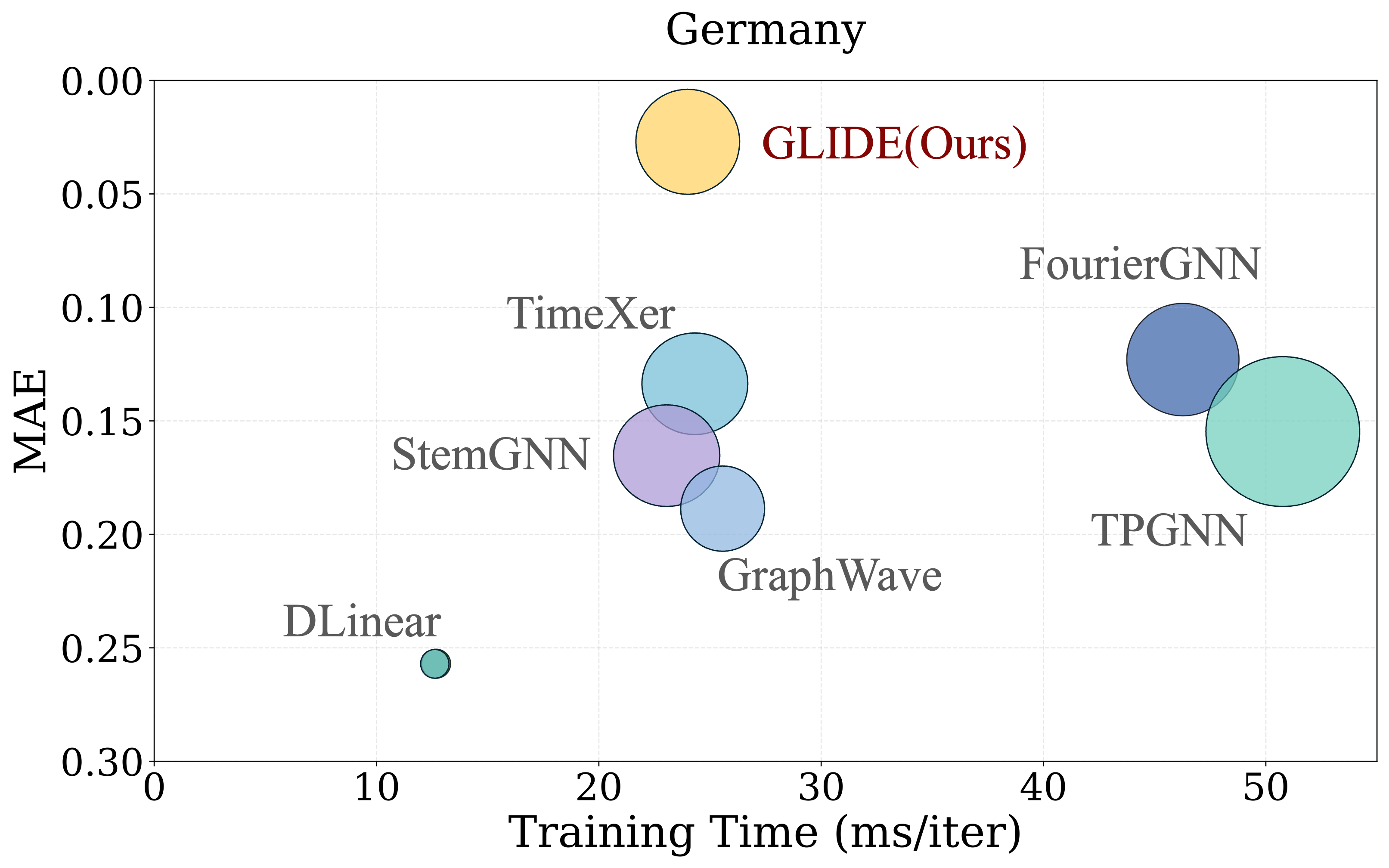}
    \caption{
    Model complexity comparison on the Germany dataset.
    Bubble areas represent GPU memory consumption during training.
    Training time and memory measurements are reported in \Cref{app:complexity},
    while MAE values are taken from \Cref{tab:germany_france}.
    }
    \label{fig:complexity}
\end{figure}
\subsection{\proposed: Putting everything together}\label{sec:glide}
We now introduce \proposed, which consists of three key components: \textbf{(C1)} a Temporal Convolution Layer, \textbf{(C2)} a \proposed Layer and \textbf{(C3)} a Prediction Layer.  We provide a detailed pipeline in \Cref{fig:overview}.
\subsubsection{C1: Temporal Convolution Layer (TC-Layer).} The TC-Layer is designed to capture a node’s piecewise-smooth temporal dynamics, combining gradual trends with abrupt changes. {We use dilated causal convolution networks, whose receptive field grows exponentially with the layer depth and dilation factor \cite{yu2016multiscalecontextaggregationdilated}.} Formally, given a 1D input sequence $\mathbf{x} \in \mathbb{R}^{T}$ and a filter $f \in \mathbb{R}^{K}$, the dilated causal convolution of $\mathbf{x}$ with $f$ at step $t$ is defined as
\begin{equation}\label{eq:dilated-conv}
(x *_d f)(t) = \sum\nolimits_{s=0}^{K-1} f(s)\,x(t - d \cdot s),
\end{equation}
where $d$ is the dilation factor that controls the spacing between sampled inputs. Stacking layers with exponentially increasing dilation factors yields an exponentially growing receptive field. The output of the TC-Layer passes through the Gating Mechanism and is then used as the feature $\mathbf{H}_i$ for topological feature aggregation in \proposed layer. 
\subsubsection{C2: \proposed Layer.} The aggregation layer updates node embeddings across $L$ layers using an adjacency decomposed into two complementary parts: (i) static $\hat{\mathbf{A}}_{\text{s}}$ for long-term topology and (ii) dynamic $\hat{\mathbf{A}}_{\text{d}}$ for transient dependencies based on the estimated graph from D1. Formally, the update rule is given by
\begin{equation}
\mathbf{H} \;=
\sigma \!
\Big( 
\underbrace{\sum\nolimits_{k=0}^K  \mathbf{w}_s^k \mathbf{A}^k_{\text{s}} \mathbf{W}_{\text{s}}}_{\text{static}} 
\Big)\mathbf{H}
+ 
\Big(
\underbrace{\sum\nolimits_{k=0}^K \mathbf{w}_d^k \mathbf{A}^k_{\text{d}}}_{\text{dynamic}}  
\mathbf{W}_{\text{d}}\Big)
\mathbf{H},
\label{eq:poly-layer}
\end{equation}
where $\mathbf{w}_k$ are polynomial coefficients, $\sigma(\cdot)$ is a nonlinearity and $\mathbf{W}$ denotes learnable parameters at layer $k$. The two components are combined additively. $\mathbf{A}_{\text{d}}$ is computed as an un-normalized average over dynamic embeddings and the polynomial expansion captures higher-order interactions. Separating static and dynamic correlations makes the model robust to dynamic correlation structures while adapting flexibly to static topology fluctuations.  
\subsubsection{C3: Prediction Layer.} The final node representations $\mathbf{H}^{(\text{f})} \in \mathbb{R}^{N \times F}$ from \proposed layer are mapped to future outputs via a 1D convolutional projection {\small $\hat{\mathbf{Y}} \;=\; \mathbf{H}^{(\text{f})}\mathbf{W}_c $} where $\mathbf{W}_c \in \mathbb{R}^{F \times P}$ is a learnable weight. Here, {\small $\hat{\mathbf{Y}} \in \mathbb{R}^{N \times P}$} represents the predicted sequences for all nodes. 
\begin{table*}[t]
\centering
\caption{\textsc{Synthetic}: Mean $\pm$ std MAE/RMSE (M/R) for Static(Syn-S) and Dynamic(Syn-D) categories with varying horizons. \proposed achieves the best MAE/RMSE across all horizons in both static and dynamic synthetic regimes.}
\label{table:synthetic_eval}
\footnotesize
\renewcommand{\arraystretch}{0.85}
\begin{tabular}{ll ccc ccc c}
\toprule
\multicolumn{2}{c}{}    
& \multicolumn{3}{c}{\textbf{Syn-S} (TCV=$\TCVEasy$)}    
& \multicolumn{3}{c}{\textbf{Syn-D} (TCV=$\TCVHard$)}  
& \textbf{Rk} \\
\cmidrule(lr){3-5}\cmidrule(lr){6-8}
Methods & Me. & 3 & 6 & 12 & 3 & 6 & 12 & \\
\midrule
\multirow{2}{*}{Informer}
  & M  & .423{\tiny$\pm$.019} & .567{\tiny$\pm$.023} & .672{\tiny$\pm$.029}
         & .613{\tiny$\pm$.031} & .772{\tiny$\pm$.036} & .897{\tiny$\pm$.063} & \multirow{2}{*}{8}\\
  & R & .552{\tiny$\pm$.020} & .726{\tiny$\pm$.025} & .878{\tiny$\pm$.026}
         & .764{\tiny$\pm$.027} & .943{\tiny$\pm$.034} & 1.045{\tiny$\pm$.047} & \\
\multirow{2}{*}{Autoformer}
  & M  & .402{\tiny$\pm$.015} & .553{\tiny$\pm$.019} & .650{\tiny$\pm$.023} 
         & .592{\tiny$\pm$.025} & .745{\tiny$\pm$.029} & .890{\tiny$\pm$.052} & \multirow{2}{*}{7}\\
  & R & .543{\tiny$\pm$.016} & .736{\tiny$\pm$.021} & .902{\tiny$\pm$.022}
         & .742{\tiny$\pm$.022} & .948{\tiny$\pm$.028} & 1.085{\tiny$\pm$.041} & \\
\multirow{2}{*}{Reformer}
  & M  & .628{\tiny$\pm$.028} & .684{\tiny$\pm$.028} & .850{\tiny$\pm$.036}
         & .621{\tiny$\pm$.032} & .689{\tiny$\pm$.027} & .722{\tiny$\pm$.051} & \multirow{2}{*}{6}\\
  & R & .903{\tiny$\pm$.032} & .958{\tiny$\pm$.033} & 1.038{\tiny$\pm$.031}
         & .927{\tiny$\pm$.033} & .979{\tiny$\pm$.035} & .972{\tiny$\pm$.044} & \\
\cmidrule(lr){1-2}\cmidrule(lr){3-5}\cmidrule(lr){6-8}
\multirow{2}{*}{FourGNN}
   & M  & .447{\tiny$\pm$.020} & .582{\tiny$\pm$.024} & .653{\tiny$\pm$.028}
          & .609{\tiny$\pm$.031} & .759{\tiny$\pm$.036} & .948{\tiny$\pm$.067} & \multirow{2}{*}{4}\\
   & R & .588{\tiny$\pm$.021} & .764{\tiny$\pm$.026} & .843{\tiny$\pm$.025}
          & .756{\tiny$\pm$.027} & .935{\tiny$\pm$.033} & 1.157{\tiny$\pm$.053} & \\
\multirow{2}{*}{StemGNN}
  & M  & .421{\tiny$\pm$.019} & .553{\tiny$\pm$.023} & .576{\tiny$\pm$.025}
         & .577{\tiny$\pm$.029} & .698{\tiny$\pm$.033} & .782{\tiny$\pm$.055} & \multirow{2}{*}{3}\\
  & R & .561{\tiny$\pm$.020} & .733{\tiny$\pm$.025} & .807{\tiny$\pm$.024}
         & .713{\tiny$\pm$.025} & .884{\tiny$\pm$.031} & .951{\tiny$\pm$.043} & \\
\cmidrule(lr){1-2}\cmidrule(lr){3-5}\cmidrule(lr){6-8}
\multirow{2}{*}{TPGNN}
   & M  & \cellcolor{RoyalBlue!10}\fixme{.251}{\tiny$\pm$.011} & \cellcolor{RoyalBlue!10}\fixme{.302}{\tiny$\pm$.013} & \fixme{.456}{\tiny$\pm$.020}
          & \fixme{.431}{\tiny$\pm$.022} & \fixme{.521}{\tiny$\pm$.025} & \fixme{.543}{\tiny$\pm$.038} & \multirow{2}{*}{2}\\
   & R & \cellcolor{RoyalBlue!10} .426{\tiny$\pm$.015} & \cellcolor{RoyalBlue!10} .459{\tiny$\pm$.016} & .589{\tiny$\pm$.018}
          & \cellcolor{RoyalBlue!10} .570{\tiny$\pm$.020} & \cellcolor{RoyalBlue!10} .580{\tiny$\pm$.021} & .643{\tiny$\pm$.029} & \\
\multirow{2}{*}{DCRNN}
   & M  & \fixme{.252}{\tiny$\pm$.011} & \fixme{.305}{\tiny$\pm$.013} & \cellcolor{RoyalBlue!10}\fixme{.358}{\tiny$\pm$.014}
          & \cellcolor{RoyalBlue!10}\fixme{.419}{\tiny$\pm$.016} & \cellcolor{RoyalBlue!10}\fixme{.472}{\tiny$\pm$.018} & \cellcolor{RoyalBlue!10}\fixme{.509}{\tiny$\pm$.022} & \multirow{2}{*}{2}\\
   & R & .434{\tiny$\pm$.015} & .469{\tiny$\pm$.016} & \cellcolor{RoyalBlue!10} .500{\tiny$\pm$.017}
          & .571{\tiny$\pm$.018} & .583{\tiny$\pm$.019} & \cellcolor{RoyalBlue!10} .612{\tiny$\pm$.023} & \\
\cmidrule(lr){1-2}\cmidrule(lr){3-5}\cmidrule(lr){6-8}
\multirow{2}{*}{\textbf{Ours}}
  & M  & \cellcolor{RoyalBlue!45} \textbf{.246}{\tiny$\pm$.008} & \cellcolor{RoyalBlue!45} \textbf{.297}{\tiny$\pm$.009} & \cellcolor{RoyalBlue!45} \textbf{.350}{\tiny$\pm$.011}
         & \cellcolor{RoyalBlue!45} \textbf{.415}{\tiny$\pm$.012} & \cellcolor{RoyalBlue!45} \textbf{.461}{\tiny$\pm$.014} & \cellcolor{RoyalBlue!45} \textbf{.496}{\tiny$\pm$.018} & \multirow{2}{*}{1}\\
  & R & \cellcolor{RoyalBlue!45} \textbf{.423}{\tiny$\pm$.010} & \cellcolor{RoyalBlue!45} \textbf{.458}{\tiny$\pm$.011} & \cellcolor{RoyalBlue!45} \textbf{.494}{\tiny$\pm$.012}
         & \cellcolor{RoyalBlue!45} \textbf{.566}{\tiny$\pm$.014} & \cellcolor{RoyalBlue!45} \textbf{.572}{\tiny$\pm$.015} & \cellcolor{RoyalBlue!45} \textbf{.602}{\tiny$\pm$.019} & \\
\bottomrule
\end{tabular}
\end{table*}

\subsection{Empirical Evaluation} \label{sec:eval}
In our empirical analysis, we aim to answer the following research questions: \textbf{(RQ1)}~In a simple synthetic setting where the data generation and time-varying topology dynamics are known, does \proposed{} improve performance? \textbf{(RQ2)}~On real-world time-series data with more difficult dynamic topology, to what extent does \proposed{} improve forecasting performance? \textbf{(RQ3)}~On commonly used real-world datasets with (near-)static topology dynamics, does \proposed{} preserve its strong performance? 
\subsubsection{Synthetic \&  Real Data.} {Following the synthetic data generation procedure described in \Cref{sec:motivating example}}, we consider two topology-change regimes: a dynamic configuration (TCV = \TCVDynamic) and a static configuration (TCV = \TCVEasy). We evaluate on widely used benchmarks, \textit{Electricity} and \textit{Solar} as in \cite{liu2022multivariate}. To further assess performance under high dynamics, we also include two recent energy datasets from \cite{shao2025data}, \emph{Germany} and \emph{France}. We also compare our method on two commonly used datasets, \textit{ETTh} and \textit{Exchange Rate} \cite{liu2022multivariate}. 
\subsubsection{Preprocessing.} We follow the preprocessing convention of \cite{liu2022multivariate,yi2023fouriergnn}: each variable is standardised per-channel using statistics computed on the training split only, and MAE/RMSE are reported in normalised space. Splits are 70/20/10 for the real-world benchmarks (Electricity, Solar, Germany, France, ETTh1, Exchange Rate) and 60/20/20 for the synthetic configurations. Full hyperparameter grids and dataset statistics are in \Cref{app:eval_details}.
\subsubsection{Baselines.} We compare against 18 representative forecasting models spanning six major methodological families. \textit{Spatial GNNs}: DCRNN \cite{li2018diffusionconvolutionalrecurrentneural}, GWaveNet \cite{wu2019graph}, MTGNN \cite{wu2020connecting}, and TPGNN \cite{liu2022multivariate}; \textit{Spectral GNNs}: StemGNN \cite{cao2021spectraltemporalgraphneural} and FourGNN \cite{yi2023fouriergnn}; \textit{Transformers}: Informer \cite{zhou2020informer}, Autoformer \cite{wu2021autoformer}, FEDformer \cite{zhou2022fedformer}, and Reformer \cite{Kitaev2020ReformerTE}; \textit{Sequence Models}: TCN \cite{Bai2018AnEE}, LSTNet \cite{lai2018modelinglongshorttermtemporal}, SFM \cite{Zhang2017StockPP}; \textit{Decomposition-based Models}: TimeMixer \cite{wang2024timemixer},  TimeXer \cite{timeXer} and PatchTST \cite{patchtst}; and the classical \textit{VAR} model and DLinear \cite{Zeng2022AreTE}. To ensure a comprehensive comparison while maintaining readability, all experiments are conducted against the same pool of 18 baselines. {The main paper reports only the strongest representative competitors due to space limitations: \Cref{table:synthetic_eval} includes the seven best-performing models, \Cref{tab:germany_france} reports 11 representative methods covering all model families, and \Cref{tab:results_classic} presents the six strongest competitors. Complete results for all 18 baselines are provided in App.\,\Cref{tab:app_results_classic}.}
\subsubsection{Hyperparameter Tuning \& Metrics.} We report Mean Absolute Error (MAE) and Root Mean Squared Error (RMSE) at horizons \(h=3,6,12\) over five random seeds for all benchmarks (\Cref{tab:synthetic-results,tab:results_classic,tab:germany_france}). 
\subsection{RQ1: Evaluation on Synthetic Benchmarks}
\label{subsec:synthetic}
{We evaluate \proposed on synthetic data spanning low-to-high topology dynamics (TCV) and compare it against seven models from three dominant categories: \emph{Spatial GNNs}, \emph{Spectral GNNs} and \emph{Transformers}.} 
Table~\ref{table:synthetic_eval} reports the mean MAE/RMSE of three prediction horizons ($h=3,6,12$). \proposed consistently outperforms all competing baselines, attaining the lowest RMSE/MAE in all six columns. This dominance persists across various topology regimes and over short to mid horizons. As the graph topology complexity (measured by $\metric$) increases, all methods exhibit performance degradation; however, \proposed maintains a significant lead. Specifically, \proposed achieves up to a \fixme{23.25\%} performance improvement over the strongest baselines \fixme{(TPGNN, h=12)}, with the most pronounced gains observed in MAE. Overall, the proposed method shows improvements across various degrees of topology dynamics for short-term prediction tasks.
\subsection{RQ2: Comprehensive Evaluation under Dynamic Topologies}
 We further validate the effectiveness of our proposed design on real-world datasets and we compare \proposed against \textbf{14 strong baselines} in \Cref{tab:germany_france}.
\begin{table*}[t]
\centering
\caption{\textsc{Real-Static-Dynamic}: Mean $\pm$ std MAE/RMSE ($h=12$). The top three results are highlighted. All standard deviations are small, $<0.009$. \proposed consistently ranks first across both low- and high-TCV real-world datasets.}
\label{tab:germany_france} 
\footnotesize 
\setlength{\tabcolsep}{3.5pt} 
\renewcommand{\arraystretch}{0.85} 
\begin{tabular}{l *{8}{c} c} 
\toprule
& \multicolumn{2}{c}{\textbf{Electricity}} 
  & \multicolumn{2}{c}{\textbf{Solar}} 
  & \multicolumn{2}{c}{\textbf{Germany}} 
  & \multicolumn{2}{c}{\textbf{France}}
  & \textbf{Rk} \\
& \multicolumn{2}{c}{\metric = \TCVElectricity} 
  & \multicolumn{2}{c}{\metric = \TCVSolar} 
  & \multicolumn{2}{c}{\metric = \TCVGermany} 
  & \multicolumn{2}{c}{\metric = \TCVFrance} 
  &  \\
\cmidrule(lr){2-3}\cmidrule(lr){4-5}\cmidrule(lr){6-7}\cmidrule(lr){8-9}
\textbf{Method} & MAE & RMSE & MAE & RMSE & MAE & RMSE & MAE & RMSE & \\
\midrule
TPGNN      & \cellcolor{RoyalBlue!10}0.055 & \cellcolor{RoyalBlue!10}0.080 & \cellcolor{RoyalBlue!10}0.123 & 0.214 & 0.099 & 0.173 & 0.089 & 0.158 & 3 \\
GWaveNet   & 0.094 & 0.140 & 0.183 & 0.238 & \cellcolor{RoyalBlue!30}\fixme{0.013} & \cellcolor{RoyalBlue!30}\fixme{0.028} & \cellcolor{RoyalBlue!30}0.012 & \cellcolor{RoyalBlue!10}0.025 & 6 \\
MTGNN      & 0.077 & 0.113 & 0.151 & 0.207 & \cellcolor{RoyalBlue!10}0.016 & \cellcolor{RoyalBlue!10}0.034 & \cellcolor{RoyalBlue!30}0.012 & \cellcolor{RoyalBlue!30}0.023 & 4 \\
\midrule
FourierGNN & \cellcolor{RoyalBlue!30}\fixme{0.051} & \cellcolor{RoyalBlue!30}0.077 & \cellcolor{RoyalBlue!30}0.120 & \cellcolor{RoyalBlue!30}0.162 & 0.110 & 0.186 & 0.096 & 0.164 & 2 \\
StemGNN    & 0.070 & 0.101 & 0.176 & 0.222 & 0.179 & 0.285 & \cellcolor{RoyalBlue!10}0.148 & 0.206 & 5 \\
\midrule
TimeMixer  & 0.091 & 0.147 & 0.166 & 0.211 & 0.181 & 0.314 & 0.167 & 0.279 & 13 \\
PatchTST   & 0.212 & 0.309 & 0.188 & 0.305 & 0.153 & 0.332 & 0.219 & 0.408 & 15 \\
\midrule
Autoformer & 0.056 & 0.083 & 0.150 & 0.193 & 0.204 & 0.376 & 0.165 & 0.263 & 8 \\
FEDformer  & \cellcolor{RoyalBlue!10}0.055 & 0.081 & 0.139 & \cellcolor{RoyalBlue!10}0.182 & 0.271 & 0.396 & 0.220 & 0.291 & 11 \\ 
Informer   & 0.070 & 0.119 & 0.151 & 0.199 & 0.283 & 0.324 & 0.137 & 0.217 & 7 \\
\midrule
TCN        & 0.057 & 0.083 & 0.176 & 0.222 & 0.187 & 0.287 & 0.172 & 0.260 & 9 \\
LSTNet     & 0.075 & 0.138 & 0.148 & 0.200 & 0.193 & 0.346 & 0.177 & 0.263 & 12 \\
\midrule
DLinear    & 0.058 & 0.092 & 0.257 & 0.313 & 0.266 & 0.368 & 0.196 & 0.259 & 16 \\
VAR        & 0.096 & 0.155 & 0.175 & 0.222 & 0.243 & 0.381 & 0.177 & 0.260 & 17 \\
\midrule
\textbf{\proposed} & \cellcolor{RoyalBlue!45}\fixme{\textbf{0.010}} & \cellcolor{RoyalBlue!45}\textbf{0.051} & \cellcolor{RoyalBlue!45}\textbf{0.024} & \cellcolor{RoyalBlue!45}\textbf{0.061} & \cellcolor{RoyalBlue!45}\fixme{\textbf{0.005}} & \cellcolor{RoyalBlue!45}\fixme{\textbf{0.020}} & \cellcolor{RoyalBlue!45}\textbf{0.010} & \cellcolor{RoyalBlue!45}\textbf{0.020} & \textbf{1} \\
\bottomrule
\end{tabular}
\end{table*}
Table~\ref{tab:germany_france} reports the mean MAE/RMSE across four datasets on {14 baseline models grouped into six architectural families}. \proposed consistently achieves state-of-the-art performance in every category ($\text{Rank}=1$), yielding a lower generalization error in terms of both MAE and RMSE. This dominance demonstrates its robustness across varying graph topology dynamics. One of the most significant margins is observed on the Electricity dataset, where \proposed reduces MAE from \fixme{0.051} to \fixme{0.010}--an absolute reduction of 0.041 and a \fixme{80\%} relative gain over the runner-up FourierGNN. On the Germany dataset, \proposed outperforms GWaveNet by \fixme{61.54\%} (MAE) and \fixme{28.57\%} (RMSE). Similarly, in France, \proposed shows a significant improvement over MTGNN.  
\begin{table*}[t]
    \centering
    \setlength{\tabcolsep}{2pt} 
    \renewcommand{\arraystretch}{0.85} 
    \caption{\textsc{Real-Benchmark}: Mean $\pm$ std RMSE/MAE (R/M) ($h=3,6,12$). The top three models for each metric are highlighted (darker shade = better rank). \proposed maintains strong performance on common benchmarks.}
    \label{tab:results_classic}
    \begin{tabular}{ll *{6}{>{\centering\arraybackslash}p{1.3cm}} c} 
      \toprule
      \multicolumn{2}{c}{} 
        & \multicolumn{3}{c}{\textbf{Exchange Rate} (\metric=\TCVExchange)}
        & \multicolumn{3}{c}{\textbf{ETTh1} (\metric=\fixme{\TCVETThOne})}
        & \multirow{2}{*}{\textbf{Rk}} \\
      \cmidrule(lr){3-5}\cmidrule(lr){6-8}
      Methods & Me. & 3 & 6 & 12 & 3 & 6 & 12 & \\
      \midrule
      \multirow{2}{*}{StemGNN}
      & R & .063{\tiny$\pm$.009} & .189{\tiny$\pm$.015} & .123{\tiny$\pm$.014} 
             & .496{\tiny$\pm$.001} & .573{\tiny$\pm$.009} & .660{\tiny$\pm$.014} & \multirow{2}{*}{7} \\
      & M  & .190{\tiny$\pm$.012} & .290{\tiny$\pm$.016} & .277{\tiny$\pm$.017}
             & .349{\tiny$\pm$.001} & .408{\tiny$\pm$.009} & .476{\tiny$\pm$.011} & \\
      \multirow{2}{*}{TimeMixer}
      & R & .217{\tiny$\pm$.014} & .277{\tiny$\pm$.018} & .293{\tiny$\pm$.020}
             & \cellcolor{RoyalBlue!10}.379{\tiny$\pm$.000} & \cellcolor{RoyalBlue!10}.459{\tiny$\pm$.003} & \cellcolor{RoyalBlue!10}.515{\tiny$\pm$.003} & \multirow{2}{*}{6} \\
      & M  & .631{\tiny$\pm$.029} & .775{\tiny$\pm$.036} & .875{\tiny$\pm$.042}
             & \cellcolor{RoyalBlue!10}.243{\tiny$\pm$.001} & \cellcolor{RoyalBlue!10}.290{\tiny$\pm$.001} & \cellcolor{RoyalBlue!10}.328{\tiny$\pm$.001} & \\
      \multirow{2}{*}{FourierGNN}
      & R & .221{\tiny$\pm$.012} & .268{\tiny$\pm$.015} & .292{\tiny$\pm$.017}
             & .508{\tiny$\pm$.004} & .558{\tiny$\pm$.002} & .611{\tiny$\pm$.007} & \multirow{2}{*}{5} \\
      & M  & \cellcolor{RoyalBlue!20}\fixme{.016}{\tiny$\pm$.003} & \cellcolor{RoyalBlue!20}\fixme{.033}{\tiny$\pm$.005} & \cellcolor{RoyalBlue!20}\fixme{.049}{\tiny$\pm$.006}
             & .351{\tiny$\pm$.004} & .385{\tiny$\pm$.002} & .425{\tiny$\pm$.007} & \\
      \cmidrule(lr){1-2}\cmidrule(lr){3-5}\cmidrule(lr){6-8}
      \multirow{2}{*}{MTGNN} 
      & R & .023{\tiny$\pm$.003} & .045{\tiny$\pm$.004} & .090{\tiny$\pm$.006} 
             & .457{\tiny$\pm$.001} & .548{\tiny$\pm$.009} & .640{\tiny$\pm$.012} & \multirow{2}{*}{4} \\
      & M  & .044{\tiny$\pm$.004} & .087{\tiny$\pm$.006} & .131{\tiny$\pm$.009} 
             & .310{\tiny$\pm$.002} & .372{\tiny$\pm$.007} & .440{\tiny$\pm$.003} & \\
      \multirow{2}{*}{TPGNN}
      & R & \cellcolor{RoyalBlue!20}.009{\tiny$\pm$.002} & \cellcolor{RoyalBlue!20}.018{\tiny$\pm$.003} & \cellcolor{RoyalBlue!10}.035{\tiny$\pm$.004}
             & .445{\tiny$\pm$.006} & .540{\tiny$\pm$.019} & .627{\tiny$\pm$.012} & \multirow{2}{*}{3} \\
      & M  & \cellcolor{RoyalBlue!10}.020{\tiny$\pm$.003} & \cellcolor{RoyalBlue!10}.039{\tiny$\pm$.004} & \cellcolor{RoyalBlue!10}.078{\tiny$\pm$.006}
             & .277{\tiny$\pm$.004} & .345{\tiny$\pm$.014} & .410{\tiny$\pm$.008} & \\
      \multirow{2}{*}{GWaveNet}
      & R & \cellcolor{RoyalBlue!10}\fixme{.013}{\tiny$\pm$.003} & \cellcolor{RoyalBlue!10}\fixme{.033}{\tiny$\pm$.004} & \cellcolor{RoyalBlue!20}\fixme{.034}{\tiny$\pm$.005}
             & \cellcolor{RoyalBlue!20}\fixme{.264}{\tiny$\pm$.004} & \cellcolor{RoyalBlue!20}\fixme{.319}{\tiny$\pm$.001} & \cellcolor{RoyalBlue!20}\fixme{.369}{\tiny$\pm$.003} & \multirow{2}{*}{2} \\
      & M  & \fixme{.078}{\tiny$\pm$.005} & \fixme{.139}{\tiny$\pm$.009} & \fixme{.124}{\tiny$\pm$.010}
             & \cellcolor{RoyalBlue!20}\fixme{.133}{\tiny$\pm$.002} & \cellcolor{RoyalBlue!20}\fixme{.164}{\tiny$\pm$.002} & \cellcolor{RoyalBlue!20}\fixme{.193}{\tiny$\pm$.002} & \\
      \cmidrule(lr){1-2}\cmidrule(lr){3-5}\cmidrule(lr){6-8}
      \multirow{2}{*}{\textbf{Ours}}
      & R & \cellcolor{RoyalBlue!30}\textbf{.007}{\tiny$\pm$.001} 
             & \cellcolor{RoyalBlue!30}\textbf{.009}{\tiny$\pm$.001} 
             & \cellcolor{RoyalBlue!30}\textbf{.012}{\tiny$\pm$.002} 
             & \cellcolor{RoyalBlue!30}\textbf{.171}{\tiny$\pm$.009} 
             & \cellcolor{RoyalBlue!30}\textbf{.184}{\tiny$\pm$.009} 
             & \cellcolor{RoyalBlue!30}\textbf{.184}{\tiny$\pm$.010}
             & \multirow{2}{*}{\textbf{1}} \\
      & M  & \cellcolor{RoyalBlue!30}\textbf{.005}{\tiny$\pm$.001} 
             & \cellcolor{RoyalBlue!30}\textbf{.006}{\tiny$\pm$.001} 
             & \cellcolor{RoyalBlue!30}\textbf{.007}{\tiny$\pm$.001}
             & \cellcolor{RoyalBlue!30}\textbf{.135}{\tiny$\pm$.007} 
             & \cellcolor{RoyalBlue!30}\textbf{.135}{\tiny$\pm$.008} 
             & \cellcolor{RoyalBlue!30}\textbf{.154}{\tiny$\pm$.009} & \\
      \bottomrule
    \end{tabular}
\end{table*}
\subsection{RQ3: Evaluation on Common Benchmarks}
\label{sec: classic eval}
 We then evaluate the performance of \proposed against {six} models drawn from three categories: \emph{SpatialGNNs}, \emph{SpectralGNNs} and \emph{SOTA: {TimeMixer}}.

Table~\ref{tab:results_classic} shows that \proposed consistently outperforms six strong baselines across both benchmarks and most forecasting horizons. It achieves the best overall ranking and attains the lowest RMSE in all evaluated settings. For MAE, \proposed yields the best performance in all but one case (ETTh1 at $h=3$), where the difference to the strongest baseline is marginal. Compared with \emph{GWaveNet}, the strongest competing method on ETTh1, \proposed reduces RMSE by up to \fixme{$50.1\%$} and MAE by up to \fixme{$20.2\%$}.
On the Exchange Rate benchmark, the improvements are more pronounced: \proposed achieves up to \fixme{$64.7\%$} RMSE reduction and up to \fixme{$\Prealmax\%$} MAE reduction over the second-best model at each horizon.
Averaged across the three horizons, \proposed reduces RMSE by \fixme{$45.6\%$} and MAE by \fixme{$78.8\%$}, confirming that the improvement is broadly sustained rather than driven by a single horizon.
\fixme{Across both Exchange Rate and ETTh1, mean RMSE reductions remain comparable (\fixme{$45.6\%$} vs \fixme{$42.5\%$}) despite a $100\times$ difference in TCV, indicating that the gains generalize across volatility regimes.}
\begin{figure*}[t]
    \centering
    \includegraphics[width=\textwidth]{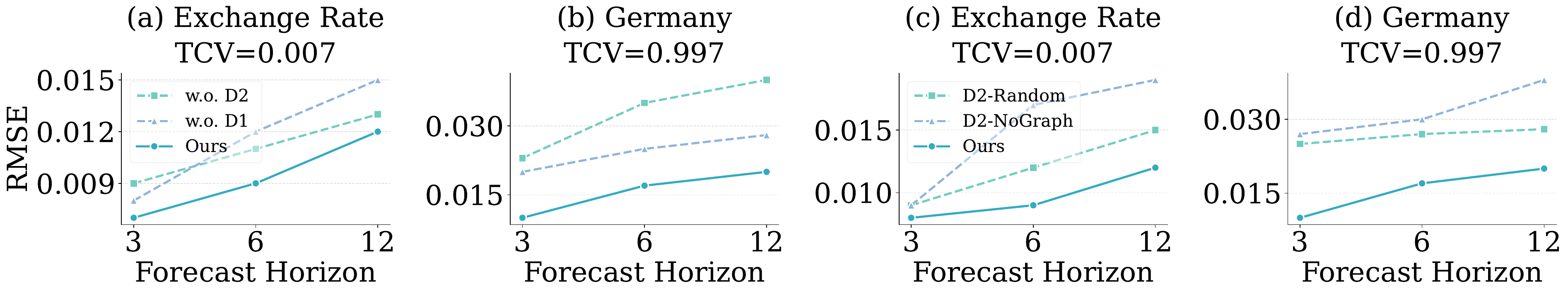}
    \caption{Ablation study of \proposed. (a--b) Comparison with base ablations (w.o. $D_1$ and $D_2$). (c--d) Sensitivity analysis of $D_2$ removing dynamic topology ($D_2$-NoGraph) and a random matrix baseline ($D_2$-Random). Both $D_1$ and $D_2$ contribute to \proposed’s gains, with $D_2$ most critical under high-TCV dynamics and learned dynamic graphs outperforming random or removed dynamic topology.}
    \label{fig:ablation_comparison}
\end{figure*}
\subsection{Ablation Studies}
\subsubsection{Impact of two Designs ($D_1, D_2$):} Figure~\ref{fig:ablation_comparison} compares \proposed with two ablated variants. Figs.~\ref{fig:ablation_comparison}(a)-(b) show that removing $D_1$ yields substantial performance degradation on both static and dynamic datasets, validating its efficacy in modeling low and high-range topological variance. Figure~\ref{fig:ablation_comparison}(b) indicates that under the high-TCV setting (Germany), $D_2$ emerges as the primary performance driver; its removal results in nearly twofold the error compared to other variants. 
\subsubsection{Impact of Dynamic Graph Designs ($D_2$):} We also compare different variants of dynamic graphs in Figs.~\ref{fig:ablation_comparison}(c)-(d), our full approach achieves the best MAE and RMSE across all horizons. Specifically: (1) replacing the dynamic graph $\mathbf{A}_d$ with a random graph ($D_2$-Random) yields slightly higher errors at longer horizons, while (2) removing the dynamic graph ($D_2$-NoGraph) produces the weakest results with substantially larger RMSE values.
\subsection{Limitations and Future Work}\label{sec:limitations}

Our analysis in \Cref{thm:d2-gnn} assumes mutually independent temporal perturbations for analytical tractability. This assumption may be violated in domains such as energy and finance, where temporal dependencies are prevalent. However, our experiments indicate that structural drift is the primary challenge under high TCV conditions, and D2 remains effective even on Germany and France (TCV $>0.96$), which exhibit substantial temporal dependence. Furthermore, TCV currently relies on Pearson correlation and therefore captures only linear relationships. Future work includes extending TCV to nonlinear similarity measures, such as Dynamic Time Warping and kernel-based methods, and investigating their ability to characterize structural dynamics.
\section{Related Work}
\label{sec:relatedwork}
\subsubsection{GNN in Time Series.} SpatialGNNs model multivariate time series as graphs whose nodes are series and whose edges encode learned relational structure \cite{wu2020connecting,yang2025benchmarkinggraphrepresentationsgraph,shao2025data}. \cite{li2018diffusionconvolutionalrecurrentneural} pioneered diffusion convolution for spatial dependencies, and \cite{wu2019graph} introduced adaptive adjacency learning to uncover hidden correlations. Subsequent work strengthens correlation modeling with spectral kernels and pure graph formulations \cite{bai2020adaptive,cao2021spectraltemporalgraphneural,yi2023fouriergnn}, and polynomial graph encoders capture higher-order dependencies \cite{liu2022multivariate}.
\subsubsection{Graph Structure Learning.} Existing GSL methods for time series fall into three categories: \textit{latent-static} approaches (e.g., MTGNN) that learn time-invariant spatial priors via node embeddings; \textit{stochastic} models (e.g., GTS) that sample probabilistic topologies via Gumbel-Softmax; and \textit{spectral operators} (e.g., FourGNN) that exploit frequency-domain kernels for efficiency. \fixme{Closer to our setting, time-varying graphical-lasso estimators infer evolving precision matrices from sliding-window covariance \cite{10.1145/3097983.3098037,kolar2010estimating}, and dynamic GNN studies have begun to connect time-varying graph properties with model behavior, e.g., through dynamic homophily in node classification~\cite{ItoKW25}. Dynamic GNN architectures such as EvolveGCN \cite{pareja2020evolvegcn}, AGCRN \cite{bai2020adaptiveNIPS}, and {DGCRN \cite{li2022dgcrn}} update node representations or adjacency over time; however, these works either provide no quantitative diagnostic for topology drift or do not explicitly decouple persistent and transient structure.}
\section{Conclusion}
We characterized the representational power of GNNs for time series of static or dynamic topology. We highlighted the limitations of current methods and proposed \proposed with theoretically grounded components. Together, our method and linear-correlation metric establish a solid foundation for dynamic topology modeling. Experiments on time series forecasting show that \proposed is robust to topology variations and outperforms strong baselines.
\subsubsection*{Acknowledgements.}
This work was supported in part by the National Science Foundation under Grants No.~IIS-2212143 and IIS-2504090, and in part by the Federal Ministry of Research, Technology and Space (BMFTR), Germany, under award number 01IS23066. 
\subsubsection*{Disclosure of AI Assistance.} The authors employed AI Tool for grammatical refinement and LaTeX table optimization. 
\bibliographystyle{splncs04}
\bibliography{mybibliography}
\newpage
\appendix
\crefalias{section}{appendix} 

\onecolumn 
\section{Notation}
\label{sec:notation}

\begin{table}[H]
\centering
\caption{Summary of notation used throughout the paper.}
\label{tab:notation}
\begin{tabular}{ll}
\toprule
\textbf{Symbol} & \textbf{Description} \\
\midrule
$N$ & Number of variables (nodes in the graph) \\
$T$ & Length of the look-back window \\
$P$ & Forecast horizon (number of future steps to predict) \\
$B$ & Length of the sliding window for node features ($B \leq T$) \\
\midrule
$\mathbf{x}_t$ & Multivariate observation at time $t$, $\mathbf{x}_t \in \mathbb{R}^N$ \\
$x_t[i]$ & The $i^{\text{th}}$ variable (node feature) in $\mathbf{x}_t$ \\
$\mathbf{X}$ & Input look-back sequence, $\mathbf{X} \in \mathbb{R}^{N \times T}$ \\
$\tilde{\mathbf{X}}$ & Row-wise standardized node signal matrix, $\tilde{\mathbf{X}} \in \mathbb{R}^{N \times T}$ \\
$\mathbf{Y}$ & Target future value(s), $\mathbf{x}_{t+1}$ or $\{\mathbf{x}_{t+1}, \dots, \mathbf{x}_{t+P}\}$ \\
\midrule
$\mathcal{G}(t)$ & Temporal graph at time $t$, $(\mathcal{V}, \mathcal{E}_t, \mathbf{X}_t, \mathbf{A}_t)$ \\
$\mathcal{V}$ & Set of vertices (variables), $|\mathcal{V}| = N$ \\
$\mathcal{E}_t$ & Set of edges representing inter-variable dependencies at time $t$ \\
$\mathbf{X}_t$ & Node feature matrix for the graph at time $t$, $\mathbf{X}_t \in \mathbb{R}^{N \times B}$ \\
$\mathbf{A}_t$ & Sparsified weighted adjacency matrix, $\mathbf{A}_t \in \mathbb{R}^{N \times N}$ \\
\midrule
$\hat{\mathbf{A}}$ & Similarity-based adjacency matrix, $\hat{\mathbf{A}} = \frac{1}{T}\tilde{\mathbf{X}}\tilde{\mathbf{X}}^{\top}$ \\
$\hat{a}_{ij}$ & Pearson correlation coefficient between nodes $v_i$ and $v_j$ \\
\midrule
$\mathcal{N}^k(v_i)$ & Set of $k$ nodes most similar to $v_i$ ($k$-NN graph neighbors) \\
$R_{ij}^{(k)}$ & $k$-hop reachability scalar between nodes $i$ and $j$, $R_{ij}^{(k)} = [\hat{\mathbf{A}}^k]_{ij}$ \\
$R_{ij}^{(\text{w})}$ & Weighted polynomial reachability, $R_{ij}^{(\text{w})} = \big[\sum_{k=0}^K \mathbf{w}^k \hat{\mathbf{A}}^k\big]_{ij}$ \\
$K$ & Maximum polynomial order in $R_{ij}^{(\text{w})}$ \\
\midrule
$f(\cdot)$ & Prediction function mapping $\mathbf{X}$ (or $\mathcal{G}(t)$) to $\mathbf{Y}$ \\
\bottomrule

\end{tabular}
\end{table}
\section{\proposed: Details of Theorems}
\label{app:proof}
\begin{theorem}

Let $\{\mathbf{A}_t\}$ be independent random adjacency structures with:
\begin{equation}
    \mathbf{A}_t \sim \mathcal{N}(0, \Sigma(t))
\end{equation}
associated with a time index $t = 0, \frac{1}{n}, \frac{2}{n}, \dots, 1$. Each $\mathbf{A}_t$ corresponds to an undirected graph $\mathcal{G}(t)$. Under the assumption that the law $\mathcal{L}(\mathbf{A}_t)$ varies smoothly with $t$, our goal is to estimate the local graph sequence:
\begin{equation}
    \mathcal{G}(0), \mathcal{G}\left(\tfrac{1}{n}\right), \dots, \mathcal{G}(1)
\end{equation}
The structure of $\mathcal{G}(t)$ is determined by the zero pattern of the precision matrix $\Sigma(t)^{-1}$. This framework applies to the global similarity evolution of the form:
\begin{equation}
    \hat{\mathbf{A}}_0 \sim \mathcal{N}(0, \Sigma(0)), \qquad \hat{\mathbf{A}}_t = \hat{\mathbf{A}}_{t-1} + \mathbf{A}_t
\end{equation}
where $\mathbf{A}_t \sim \mathcal{N}(0, \Sigma(t))$.

In the i.i.d. has considered $\ell_1$-penalized maximum likelihood estimation over the space of positive definite matrices. Specifically, the estimator is defined as
\begin{equation}
\widehat{\Sigma}_n
=
\arg\min_{\Sigma \succ 0}
\left\{
\operatorname{tr}\!\left(\Sigma^{-1}\widehat{S}_n\right)
+ \log |\Sigma|
+ \lambda \lVert \Sigma^{-1} \rVert_1
\right\},
\label{eq:iid-glasso}
\end{equation}
where $\widehat{S}_n$ denotes the sample covariance matrix.

In the non-i.i.d.\ setting, our approach is to estimate $\Sigma(t)$ at time $t$ by
\begin{equation}
\widehat{\Sigma}_n(t)
=
\arg\min_{\Sigma \succ 0}
\left\{
\operatorname{tr}\!\left(\Sigma^{-1}\widehat{S}_n(t)\right)
+ \log |\Sigma|
+ \lambda \lVert \Sigma^{-1} \rVert_1
\right\},
\label{eq:tv-glasso}
\end{equation}
where
\begin{equation}
\widehat{S}_n(t)
=
\frac{\sum_s w_{st} Z_s Z_s^\top}{\sum_s w_{st}}
\label{eq:weighted-cov}
\end{equation}
is a weighted covariance matrix. The weights are given by
$w_{st}
=
K\!\left( \frac{|s - t|}{h_n} \right),$
where $K(\cdot)$ is a symmetric, nonnegative kernel function over time and $h_n$ is a bandwidth parameter. In other words, $\widehat{S}_n(t)$ is the kernel estimator of the covariance matrix at time $t$.

\end{theorem}

\begin{proof}
\label{sec:dev-bound}
In this section, we establish the probabilistic bounds governing the convergence of our similarity-based graph estimators. Our data consists of $n$ multivariate observations $\mathbf{x}_k \in \mathbb{R}^N$ sampled at $x = 0, 1/n, \dots, 1$. These observations correspond to the independent random vectors $\mathbf{x}_k \sim \mathcal{N}(0, \Sigma_k)$. We use the notion \textit{Local Similarity Graph approximation} $\mathbf{A}_k$ as the instantaneous graph structure at time index $k$. The \textit{Global Similarity Graph} $\hat{\mathbf{A}}$ is then reconstructed via a kernel-weighting scheme that prioritizes observations proximal to $x_0$. The motive is to generalize the usual inequalities and use moment generating functions to show that for the estimated covariance matrix with the constraint of window $t_0$ is close enough to true topology. 

\paragraph{Bounds for Kernel-Smoothed Similarity Graphs}
\label{sec:kernel-dev}

We derive large deviation inequalities for the empirical covariance matrix based on kernel regression. We utilize a symmetric, non-negative \fixme{kernel function $K$} with bounded support $[-1, 1]$. To estimate the Global Similarity Graph at $x_0$, we apply a weighting scheme $\ell_k(x_0)$ that assigns higher importance to Local Similarity approximations near the forecast horizon. Let $x_k = \frac{t_0 - k}{n}$. The kernel weights are defined as:
\begin{equation}
\label{eq::kernel-weight}
\ell_k(x_0) = \frac{2}{nh} K\left(\frac{x_k - x_0}{h}\right) \approx \frac{K\left(\frac{x_k - x_0}{h}\right)}{\sum_{k=1}^n K\left(\frac{x_k - x_0}{h}\right)}
\end{equation}
where $h$ is the bandwidth controlling the temporal smoothness of the Global Graph. The expected Global Similarity structure is denoted by $\Phi_1(i, j) = \mathbb{E}[\sum \ell_k \mathbf{x}_{ki}\mathbf{x}_{kj}]$. We decompose the estimation error into a bias term and a variance (deviation) term:
\begin{equation}
\label{eq::bd-decompose}
|\hat{a}_{ij} - \sigma_{ij}(x_0)| \leq |\hat{a}_{ij} - \Phi_1(i,j)| + |\Phi_1(i,j) - \sigma_{ij}(x_0)|
\end{equation}
where $\hat{a}_{ij}$ is the $ij$-th entry of the Global Similarity Graph $\hat{\mathbf{A}}$.

\begin{lemma}[Global Bias Bound]
\label{lemma::bias}
Suppose there exists a constant $C > 0$ such that the local precision structures vary smoothly, satisfying $\max_{i,j} \sup_x |\sigma''(x, i, j)| \leq C$. Then, for all $x \in [0, 1]$, the bias of the Global Similarity Graph $\hat{\mathbf{A}}$ relative to the true local state $\Sigma(x)$ is bounded by:
\begin{equation}
    \max_{i,j} \left| \mathbb{E}[\hat{a}_{ij}(x)] - \sigma_{ij}(x) \right| = O(h).
\end{equation}
\end{lemma}

\begin{proof}
Without loss of generality, let $x = x_0$ represent the current forecast point. The expected value of an entry in the Global Similarity Graph is given by $\mathbb{E}[\hat{a}_{ij}(x_0)] = \Phi_1(i,j)$. We approximate the weighted sum of local similarity approximations using a Riemann integral:

\begin{align*}
\Phi_1(i,j) &= \frac{1}{n} \sum_{k=1}^n \frac{2}{h} K\left(\frac{x_k - x_0}{h}\right) \sigma_{ij}(x_k) \\
&\approx \int_{x_n}^{x_0} \frac{2}{h} K\left(\frac{u - x_0}{h}\right) \sigma_{ij}(u) \, du \\
&= 2 \int_{-1/h}^{0} K(v) \sigma_{ij}(x_0 + hv) \, dv.
\end{align*}

We apply Taylor's Formula to the local similarity structure $\sigma_{ij}(x_0 + hv)$, expanding it around the state at $x_0$:
\begin{equation*}
2 \int_{-1}^{0} K(v) \left( \sigma_{ij}(x_0) + hv \sigma'_{ij}(x_0) + \frac{\sigma''_{ij}(y(v))(hv)^2}{2} \right) dv
\end{equation*}
where $y(v)$ is a point between $x_0$ and $x_0 + hv$. Extracting the term $\sigma_{ij}(x_0)$ and using the kernel property $2\int_{-1}^0 K(v)dv = 1$, we obtain:
\begin{equation*}
= \sigma_{ij}(x_0) + 2 \int_{-1}^{0} K(v) \left( hv \sigma'_{ij}(x_0) + \frac{C(hv)^2}{2} \right) dv.
\end{equation*}

Evaluating the remaining integral:
\begin{align*}
&2 h \sigma'_{ij}(x_0) \int_{-1}^{0} v K(v) \, dv + \frac{Ch^2}{2} \int_{-1}^{0} v^2 K(v) \, dv \\
&\leq h \sigma'_{ij}(x_0) + \frac{C h^2}{4}.
\end{align*}

This demonstrates that the expected Global Similarity Graph consists of the true local structure at $x_0$ plus an error term determined by the bandwidth $h$. Therefore:
\begin{equation*}
\Phi_1(i,j) - \sigma_{ij}(x_0) = O(h).
\end{equation*}
\end{proof}
\begin{lemma}[Large Deviation of Graph Structures]
\label{lemma:deviation}
For a sufficiently small $\epsilon$, the probability that the empirical Global Similarity Graph $\hat{\mathbf{A}}$ deviates significantly from its expected value is exponentially bounded:
\begin{equation}
\mathbb{P}(|\hat{a}_{ij} - \mathbb{E}[\hat{a}_{ij}]| > \epsilon) \leq \exp\left\{ - C n h \epsilon^2 \right\}
\end{equation}
\end{lemma}
We give a simplified proof as follows and full proof later.

Let $A_k = \mathbf{x}_{ki} \mathbf{x}_{kj} - \sigma_{ij}(x_k)$ represent the noise in the $k$-th Local Similarity approximation. Using Markov's inequality and the Moment Generating Function (MGF) for Gaussian products:
\begin{equation}
\mathbb{P}\left(\sum_{k=1}^n \ell_k(x_0) A_k > \epsilon\right) \leq \frac{\mathbb{E}[e^{t\sum \ell_k A_k}]}{e^{nt\epsilon}}
\end{equation}
By optimizing the parameter $t = \frac{\epsilon}{4\Phi_2}$, where $\Phi_2$ is the second-order kernel moment, we obtain the tail bound. This establishes that the Global Similarity Graph $\hat{\mathbf{A}}$ is a \textit{persistent} estimator of the underlying precision structure $\Sigma(x_0)^{-1}$ as $nh \to \infty$.

\begin{proof}
\label{app:proof_full}
Let $A_k = \mathbf{x}_{ki} \mathbf{x}_{kj} - \sigma_{ij}(x_k)$ represent the centered noise of the $k$-th Local Similarity approximation. We aim to bound the probability that the Global Similarity entry $\hat{a}_{ij}$ deviates from its expectation:
\begin{equation*}
\mathbb{P}\left(|\hat{a}_{ij} - \mathbb{E}[\hat{a}_{ij}]| > \epsilon\right) = \mathbb{P}\left(\sum_{k=1}^n \ell_k(x_0) A_k > \epsilon\right)
\end{equation*}

By applying Markov's inequality to the exponential of the weighted sum for $t > 0$:
\begin{equation}
\label{eq:kernel-markov}
\mathbb{P}\left(\sum_{k=1}^n \ell_k(x_0) A_k > \epsilon\right) \leq \frac{\mathbb{E}\left[\exp\left(t \sum_{k=1}^n \frac{2}{h} K(\frac{x_k - x_0}{h}) A_k\right)\right]}{e^{nt\epsilon}}
\end{equation}

To evaluate the expectation, we define the following local moments at time $x_k$:
\begin{itemize}
    \item $a_k, b_k = \frac{2t}{h} K(\frac{x_k - x_0}{h}) (\sigma_i \sigma_j \pm \sigma_{ij})$
    \item $\Phi_m = \frac{1}{n} \sum_{k=1}^n \text{poly}_m(a_k, b_k)$, where $\Phi_2 \propto \frac{1}{h}$ represents the global graph variance.
    \item $M = \max_k \{ \frac{2}{h} K(\frac{x_k - x_0}{h}) \sigma_i(x_k) \sigma_j(x_k) \}$
\end{itemize}

Using the Taylor expansion for the log-moment generating function, we bound the joint expectation of the independent local approximations:
\begin{align*}
\ln \mathbb{E}\left[e^{t \sum \text{Weights} \cdot A_k}\right] &\leq \sum_{k=1}^n \left( -nt\Phi_1 + \frac{1}{2} \ln \frac{1}{(1-a_k)(1+b_k)} \right) \\
&\leq nt^2 \Phi_2 + nt^3 \Phi_3 + \frac{9}{5} nt^4 \Phi_4
\end{align*}

Substituting this back into the probability bound \eqref{eq:kernel-markov} and choosing the optimal step size $t = \frac{\epsilon}{4\Phi_2}$, we simplify the exponent:
\begin{align*}
\mathbb{P}(\dots) &\leq \exp\left( -nt\epsilon + nt^2 \Phi_2 + nt^3 \Phi_3 + \frac{9}{5} nt^4 \Phi_4 \right) \\
&\leq \exp\left( \frac{-n\epsilon^2}{4\Phi_2} \left( 1 - \text{rem}(\epsilon) \right) \right) \\
&\leq \exp\left( -\frac{3nh\epsilon^2}{20 C_1 (\sigma_i^2 \sigma_j^2 + \sigma_{ij}^2)} \right)
\end{align*}

This final inequality relies on the fact that $\Phi_2 \approx \frac{C_1 (\sigma_i^2 \sigma_j^2 + \sigma_{ij}^2)}{h}$. Thus, provided the deviation $\epsilon$ satisfies the regularity condition $\epsilon \leq \frac{\Phi_2}{M}$, the Global Similarity Graph converges exponentially at the rate $nh$.
\end{proof}

We now present a comprehensive theorem that summarizes all the proofs. 
\begin{theorem}[Dynamic Topology Identification]
\label{thm:d2-gnn-full}
Consider a multivariate time series $\mathbf{X} \in \mathbb{R}^{N \times T}$ with an underlying graph topology evolving as a latent process $\hat{\mathbf{A}}_t = \hat{\mathbf{A}}_{t-1} + \mathbf{A}_t$, where the local perturbations follow $\mathbf{A}_t \sim \mathcal{N}(0, \boldsymbol{\Sigma}(t))$. Let the precision matrix be denoted by $\Theta(t) = \boldsymbol{\Sigma}(t)^{-1}$. Under the following conditions:

\begin{enumerate}
    \item \textbf{Temporal Smoothness:} The covariance $\boldsymbol{\Sigma}(t)$ is twice differentiable with respect to $t$, such that $\max_{i,j} \sup_t |\sigma''_{ij}(t)| \leq C$ for some $C > 0$.
    \item \textbf{Sparsity:} The number of active edges $|E_t| = s$ satisfies $(N + s) = o(n^{2/3}/\log N)$.
    \item \textbf{Optimal Bandwidth:} The kernel weights $w_{st} = K(\frac{|s-t|}{h})$ utilize a bandwidth $h \asymp n^{-1/3}$.
\end{enumerate}

The time-varying topology at time $t$ is identified via the $\ell_1$-regularized objective:
\begin{equation}
    \widehat{\Theta}_t = \arg\min_{\Theta \succ 0} \left\{ \text{tr}(\Theta \widehat{\mathbf{A}}_t) - \log|\Theta| + \lambda_n \|\Theta\|_1 \right\}
\end{equation}
where $\widehat{\mathbf{A}}_t = \frac{\sum_s w_{st} \mathbf{A}_s}{\sum_s w_{st}}$ is the kernel-weighted Global Similarity Graph. Then, with probability $1 - o(1)$, the estimator $\widehat{\Theta}_t$ is \textbf{sparsistent}, correctly recovering the zero-pattern of $\Theta(t)$.
\end{theorem}

\end{proof}
\section{Time Complexity Analysis}
\label{app:complexity}
\noindent With a given sequence $\mathbf{X} \in \mathbb{R}^{N \times T}$, let $F$ denote the hidden feature dimension, $|\mathcal{E}|$ the number of edges, \fixme{$k$ the kernel size} and $T_p$ the prediction horizon. The temporal convolution stage (S1) takes $O(N k T \log T)$.  The \proposed (S2) takes $O(k |E| F)$ \fixme{for order $k$}.  The prediction stage (S3) requires $O(N F T_p)$.  Overall, the total complexity is $O(N k T \log T +  k |E| F + N F T_p)$, with memory cost dominated by $O(NF)$ for embeddings. \Cref{fig:complexity} shows that \proposed achieves better efficiency than FourierGNN and TPGNN and achieves a significant performance gain.

\section{Extended Related Work}
\label{app:related work}
\subsubsection{Graph Topology Identification.}  When the graph structure in time series is not available, topology identification is essential for understanding the network. In static settings, under mild assumptions (e.g., Gaussian), the problem reduces to estimating the sample covariance matrix, which can be interpreted as an adjacency matrix with self-loops \cite{wu2020connecting}. Structural equation models (SEM) formulate node values as linear functions of neighbors with noise, enabling adjacency estimation via regression. With prior constraints such as smoothness or sparsity, this leads to Graph LASSO formulations \cite{10.1145/3097983.3098037}. For dynamic settings, empirical covariance is updated via exponentially-weighted moving average~\cite{natali2022learningtimevaryinggraphsonline}. 
\section{Empirical Setup}
\label{app:eval_details}

\subsection{Evaluation on Synthetic Benchmark}
\label[appendix]{app:eval_syn}
\subsubsection{Hyperparameter Tuning.} We only report a subset of the tuned hyper-parameters, which include but are not limited to: learning rate $10^{(-2\sim-4)}$, hidden embedding size $2^{8\sim10}$, batch size $2^{4\sim9}$, number of encoder and decoder layers $1\sim3$, number of attention heads $2\sim8$, kernel sizes $3\sim7$ and rolling window sizes $12\times(1\sim3)$. For the electricity, solar and ETT datasets, we report the best results obtained from \cite{yi2023fouriergnn,wu2023timesnet,wu2020connecting}.

\subsubsection{Experiment Setting of Motivation Example} Each dataset has a length of 2400, with a 60/20/20 train/validation/test split. We utilize \metric to quantify the correlation dynamics within the dataset and employ Mean Absolute Error (MAE) as the evaluation metric, specifically reporting results for a forecasting horizon of $H=3$. 
\paragraph{Details of Synthetic Data Generation} At time step $t$, the signal $X(t) \in \mathbb{R}^{N \times 1}$ is sampled as $X(t) \sim \mathcal{N}\!\big(A(t-1)X(t-1), \, \sigma \big)$, where $\mathcal{N}$ denotes the normal distribution, $\sigma \in \mathbb{R}$ specifies the variance, and $X(0)$ is drawn from a discrete uniform distribution, $A(t-1)$ denotes a transition matrix indicates the underlying graph topology behind time series. We construct $N_w$ transition graphs $(\mathbf{A}_{1}, \dots, \mathbf{A}_{N_w})$, each encoding the structural and transitional statistics of the series within a specific interval. Random walk sequences are generated on these graphs, forming the temporal dynamics: variables correspond to independent walks but share the same transition graph within an interval. {Given a cycle length $T_\mathcal{G} \geq N_w$, we divide it into $N_w$ equal subintervals of length $T_s = T_\mathcal{G} / N_w$.} Each subinterval is assigned a transition matrix, {$\mathbf{A}(t) = G\!\left(\Big\lfloor \tfrac{t \bmod T_\mathcal{G}}{T_s} \Big\rfloor \right)$}, with $\lfloor \cdot \rfloor$ denoting the floor function. Within each $T_s$-interval, we simulate $M$ independent random walks as the $M$-dimensional time series. {The static limit corresponds to $N_w = 1$ (a single fixed transition graph), while smaller $T_s$ (equivalently, larger $N_w$ at fixed $T_\mathcal{G}$) induces faster switching and higher-order correlations.} To further increase stochasticity and realism, we incorporate random rewiring and Bernoulli noise. We highlight two extreme regimes of this construction and generate six multivariate time series datasets of length $2400$. Each dataset is split into training, validation, and test subsets with a $7{:}1{:}2$ ratio. Full algorithmic details, we provide algorithm table in \Cref{algorithm:syn}. 
\begin{algorithm}
\caption{Generating MTS Data with the NPR Model}
\label{algorithm:syn}
\small
\begin{algorithmic}[1]
\Require Total length $T$ of the MTS data, number of variables $N$, number of constant matrices $N_w$
\Require Cycle length $T_\mathcal{G}$, standard deviation $\sigma$, matrix sparsity threshold $\delta$
\Ensure Synthetic MTS data $\mathbf{X} \in \mathbb{R}^{T \times N}$

\State Generate random orthogonal matrix $\mathbf{P} \in \mathbb{R}^{N \times N}$
\For{$i = 1$ to $N_w$}
    \State $\mathbf{\Sigma}_i = \text{diag}(|\mathcal{N}(0,1)|, \dots, |\mathcal{N}(0,1)|)$
    \State $\mathbf{A}_i = \mathbf{P}^\top \mathbf{\Sigma}_i \mathbf{P}$, $\alpha = 0$
    \State $\mathbf{A}_i[\mathbf{G}_i < \alpha] \gets 0$
    \While{sparsity$(\mathbf{G}_i) > \delta$}
        \State $\alpha \gets \alpha + 0.02$
        \State $\mathbf{A}_i[\mathbf{A}_i < \alpha] \gets 0$
    \EndWhile
    \State $\mathbf{A}_i \gets$ symmetric normalized Laplacian of $\mathbf{G}_i$
\EndFor

\State Initialize $\mathbf{X} \gets \mathbf{0} \in \mathbb{R}^{T \times N}$
\State $T_s \gets T_\mathcal{G} / N_w$
\For{$t = 1$ to $T$}
    \If{$(t - 1) \bmod T_\mathcal{G} = 0$}
        \State Initialize $\mathbf{x} \in \mathbb{R}^N$ randomly from $\{-1, -0.5, 0.5, 1\}$
    \Else
        \State $\mathbf{x} \sim \mathcal{N}\left(\mathbf{A}_{\lfloor ((t - 1) \bmod T_\mathcal{G}) / T_s \rfloor} \cdot \mathbf{X}[t - 1], \sigma\right)$
    \EndIf
    \State $\mathbf{X}[t] \gets \mathbf{x}$
\EndFor
\end{algorithmic}
\end{algorithm}

\subsubsection{Implementation} All models are trained under a standardized setup with temporal splits of 60\%/20\%/20\%, and early stopping of epoch 5 is applied based on validation Mean Absolute Error (MAE). We adopt the same hyperparameter grids for parameters embedding sizes $\{32,128,512\}$, hidden sizes $\{16,64,128\}$, learning rates $\{10^{-3},10^{-4},10^{-5}\}$, and batch sizes $\{32,64,128,512\}$.  

\subsection{More Details of RQ2}
\subsubsection{Hyperparameter Tuning.} We performed a grid search over a wide range of hyperparameter values for all reported models. Due to limited space, we only report a subset of the tuned hyperparameters, which include but are not limited to: learning rate $10^{(-2\sim-4)}$, hidden embedding size $2^{8\sim10}$, batch size $2^{4\sim9}$, number of encoder and decoder layers $1\sim3$, number of attention heads $2\sim8$, kernel sizes $3\sim7$ and rolling window sizes $12\times(1\sim3)$. For the baseline result, we report the best results obtained from \cite{yi2023fouriergnn,wu2023timesnet,wu2020connecting}.

\begin{table*}[h]
\centering
\caption{Average MAE and RMSE (mean $\pm$ std) on static and dynamic benchmarks. Standard deviations are computed over 5 random seeds. Across all 18 baselines, \proposed achieves the best aggregate rank on real-world static and dynamic datasets.}
\label{tab:app_results_classic}
\setlength{\tabcolsep}{1pt} 
\renewcommand{\arraystretch}{1}
\resizebox{\linewidth}{!}{%
\begin{tabular}{llccccccccc} 
\toprule
& & \multicolumn{2}{c}{\textbf{Electricity}} 
  & \multicolumn{2}{c}{\textbf{Solar}} 
  & \multicolumn{2}{c}{\textbf{Germany}} 
  & \multicolumn{2}{c}{\textbf{France}}
  & \textbf{Rk} \\
\cmidrule(lr){3-4}\cmidrule(lr){5-6}\cmidrule(lr){7-8}\cmidrule(lr){9-10}
\textbf{Family} & 
& MAE & RMSE 
& MAE & RMSE 
& MAE & RMSE 
& MAE & RMSE 
& \\
\cmidrule(lr){1-2}\cmidrule(lr){3-10}

\multirow{3}{*}{\textbf{SpatialGNNs}}
& TPGNN      
& \cellcolor{RoyalBlue!10}.055{\tiny$\pm$.001} & \cellcolor{RoyalBlue!10}.080{\tiny$\pm$.002} 
& \cellcolor{RoyalBlue!10}.123{\tiny$\pm$.003} & .214{\tiny$\pm$.005} 
& .099{\tiny$\pm$.002} & .173{\tiny$\pm$.004} 
& .089{\tiny$\pm$.002} & .158{\tiny$\pm$.004} 
& 3 \\
& GWaveNet   
& .094{\tiny$\pm$.002} & .140{\tiny$\pm$.003} 
& .183{\tiny$\pm$.004} & .238{\tiny$\pm$.004} 
& \cellcolor{RoyalBlue!30}.013{\tiny$\pm$.001} & \cellcolor{RoyalBlue!30}.028{\tiny$\pm$.001} 
& \cellcolor{RoyalBlue!30}.012{\tiny$\pm$.001} & \cellcolor{RoyalBlue!10} .025{\tiny$\pm$.002} 
& 6 \\
& MTGNN        
& .077{\tiny$\pm$.001} & .113{\tiny$\pm$.002} 
& .151{\tiny$\pm$.003} & .207{\tiny$\pm$.005} 
& \cellcolor{RoyalBlue!10}.016{\tiny$\pm$.001} & \cellcolor{RoyalBlue!10}.034{\tiny$\pm$.002} 
& \cellcolor{RoyalBlue!30}.012{\tiny$\pm$.001} & \cellcolor{RoyalBlue!30}.023{\tiny$\pm$.002} 
& 4 \\
\cmidrule(lr){1-2}\cmidrule(lr){3-10}

\multirow{2}{*}{\textbf{SpectralGNNs}}
& FourierGNN 
& \cellcolor{RoyalBlue!30}.051{\tiny$\pm$.001} & \cellcolor{RoyalBlue!30}.077{\tiny$\pm$.002} 
& \cellcolor{RoyalBlue!30}.120{\tiny$\pm$.004} & \cellcolor{RoyalBlue!30}.162{\tiny$\pm$.006} 
& .110{\tiny$\pm$.003} & .186{\tiny$\pm$.005} 
& .096{\tiny$\pm$.002} & .164{\tiny$\pm$.004} 
& 2 \\
& StemGNN 
& .070{\tiny$\pm$.002} & .101{\tiny$\pm$.003} 
& .176{\tiny$\pm$.006} & .222{\tiny$\pm$.008} 
& .179{\tiny$\pm$.005} & .285{\tiny$\pm$.009} 
& \cellcolor{RoyalBlue!10}.148{\tiny$\pm$.004} & .206{\tiny$\pm$.007} 
& 5 \\
\cmidrule(lr){1-2}\cmidrule(lr){3-10}

\multirow{3}{*}{\textbf{Decomposition}}
& TimeMixer  
& .091{\tiny$\pm$.000} & .147{\tiny$\pm$.000} 
& .166{\tiny$\pm$.001} & .211{\tiny$\pm$.002} 
& .181{\tiny$\pm$.005} & .314{\tiny$\pm$.007} 
& .167{\tiny$\pm$.004} & .279{\tiny$\pm$.006} 
& 13 \\
& TimeXer    
& .198{\tiny$\pm$.001} & .302{\tiny$\pm$.001} 
& .139{\tiny$\pm$.002} & .243{\tiny$\pm$.003} 
& .143{\tiny$\pm$.006} & .315{\tiny$\pm$.008} 
& .210{\tiny$\pm$.007} & .395{\tiny$\pm$.009} 
& 14 \\
& PatchTST   
& .212{\tiny$\pm$.001} & .309{\tiny$\pm$.000} 
& .188{\tiny$\pm$.001} & .305{\tiny$\pm$.002} 
& .153{\tiny$\pm$.004} & .332{\tiny$\pm$.006} 
& .219{\tiny$\pm$.005} & .408{\tiny$\pm$.008} 
& 15  \\
\cmidrule(lr){1-2}\cmidrule(lr){3-10}

\multirow{3}{*}{\textbf{Transformer}}
& Autoformer 
& .056{\tiny$\pm$.001} & .083{\tiny$\pm$.003} 
& .150{\tiny$\pm$.002} & .193{\tiny$\pm$.002} 
& .204{\tiny$\pm$.001} & .376{\tiny$\pm$.002} 
& .165{\tiny$\pm$.001} & .263{\tiny$\pm$.002} 
& 8 \\
& FEDformer  
& \cellcolor{RoyalBlue!10}.055{\tiny$\pm$.000} & .081{\tiny$\pm$.001} 
& .139{\tiny$\pm$.001} & \cellcolor{RoyalBlue!10}.182{\tiny$\pm$.001} 
& .271{\tiny$\pm$.002} & .396{\tiny$\pm$.001} 
& .220{\tiny$\pm$.001} & .291{\tiny$\pm$.002} 
& 11 \\ 
& Informer   
& .070{\tiny$\pm$.002} & .119{\tiny$\pm$.001} 
& .151{\tiny$\pm$.002} & .199{\tiny$\pm$.001} 
& .283{\tiny$\pm$.003} & .324{\tiny$\pm$.004} 
& .137{\tiny$\pm$.003} & .217{\tiny$\pm$.009} 
& 7 \\
\cmidrule(lr){1-2}\cmidrule(lr){3-10}
\multirow{3}{*}{\textbf{Sequence}}
& SFM
& .086{\tiny$\pm$.002} & .129{\tiny$\pm$.003} 
& .161{\tiny$\pm$.003} & .283{\tiny$\pm$.006} 
& .184{\tiny$\pm$.002} & .245{\tiny$\pm$.005} 
& .156{\tiny$\pm$.003} & .238{\tiny$\pm$.005} 
& 10 \\
& TCN          
& .057{\tiny$\pm$.001} & .083{\tiny$\pm$.002} 
& .176{\tiny$\pm$.004} & .222{\tiny$\pm$.004} 
& .187{\tiny$\pm$.001} & .287{\tiny$\pm$.006} 
& .172{\tiny$\pm$.002} & .260{\tiny$\pm$.002} 
& 9 \\
& LSTNet      
& .075{\tiny$\pm$.002} & .138{\tiny$\pm$.003} 
& .148{\tiny$\pm$.003} & .200{\tiny$\pm$.004} 
& .193{\tiny$\pm$.004} & .346{\tiny$\pm$.007} 
& .177{\tiny$\pm$.004} & .263{\tiny$\pm$.003} 
& 12 \\
\cmidrule(lr){1-2}\cmidrule(lr){3-10}

\multirow{2}{*}{\textbf{Others}}
& DLinear     
& .058{\tiny$\pm$.001} & .092{\tiny$\pm$.002} 
& .257{\tiny$\pm$.005} & .313{\tiny$\pm$.006} 
& .266{\tiny$\pm$.003} & .368{\tiny$\pm$.007} 
& .196{\tiny$\pm$.004} & .259{\tiny$\pm$.006} 
& 16 \\
& VAR          
& .096{\tiny$\pm$.002} & .155{\tiny$\pm$.003} 
& .175{\tiny$\pm$.004} & .222{\tiny$\pm$.004} 
& .243{\tiny$\pm$.005} & .381{\tiny$\pm$.008} 
& .177{\tiny$\pm$.004} & .260{\tiny$\pm$.005} 
& 17 \\
\midrule
\textbf{Ours}
& \textbf{\proposed} 
& \cellcolor{RoyalBlue!45}\textbf{.010}{\tiny$\pm$\textbf{.001}} & \cellcolor{RoyalBlue!45}\textbf{.051}{\tiny$\pm$\textbf{.002}}
& \cellcolor{RoyalBlue!45}\textbf{.024}{\tiny$\pm$\textbf{.003}} & \cellcolor{RoyalBlue!45}\textbf{.061}{\tiny$\pm$\textbf{.005}}
& \cellcolor{RoyalBlue!45}\textbf{.005}{\tiny$\pm$\textbf{.001}} & \cellcolor{RoyalBlue!45}\textbf{.020}{\tiny$\pm$\textbf{.003}}
& \cellcolor{RoyalBlue!45}\textbf{.010}{\tiny$\pm$\textbf{.002}} & \cellcolor{RoyalBlue!45}\textbf{.020}{\tiny$\pm$\textbf{.004}}
& \textbf{1} \\
\bottomrule
\end{tabular}
}
\end{table*}
\subsection{More Details of RQ3}
\label{app:details}
\subsubsection{Dataset \& Metrics. } We finally evaluate \proposed on a real-world high-TCV data set. Specifically, we select a subset from energy datasets in \cite{shao2025data}, \emph{Germany}, \emph{France} and two public datasets (electricity and solar). It records electricity production in Germany and France in all generation types for all available years. We use the consistent data splits (70\%/20\%/10\%) for all models and report \emph{Mean Absolute Error (MAE)} and \emph{Root Mean Squared Error (RMSE)} of forecasting horizons with 12 time steps on five random seeds with enabled early stop. The datasets span a wide correlation spectrum. For instance, Germany and France exhibit larger magnitudes of \metric{} (near 1) compared to the classic datasets (around 0.5), indicating stronger temporal fluctuations and heterogeneous regional dynamics. For other results on dynamic datasets please refer to \Cref{tab:energy}. 
\subsubsection{Hyperparameter Tuning.} We performed a grid search over a wide range of hyperparameter values for all reported models. We report the tuned hyper-parameters, which include but are not limited to: learning rate $10^{(-2\sim-4)}$, hidden embedding size $2^{8\sim10}$, batch size $2^{4\sim9}$, number of encoder and decoder layers $1\sim3$, number of attention heads $2\sim8$, kernel sizes $3\sim7$ and rolling window sizes $12\times(1\sim3)$. For the electricity, solar and ETT datasets, we report the best results obtained from \cite{yi2023fouriergnn,wu2023timesnet}. 
\begin{table}[H]
\centering
\caption{Dataset statistics and corresponding TCV values.}
\label{tab:real_datasets}
\begin{tabular}{lcccccc}
\toprule
Dataset & $N$ & Rows & Windows & $B$ & Step & TCV  \\
\midrule
Germany  & 16  & 333{,}060 & 333{,}037 & 24 & 1 & \TCVGermany \\
France   & 10  & 83{,}265  & 83{,}242  & 24 & 1 & \TCVFrance \\
Exchange Rate         & 8   & 7{,}588   & 7{,}565   & 24 & 1 & \TCVExchange \\
Electricity           & 321 & 26{,}304  & 26{,}281  & 24 & 1 & \TCVElectricity \\
ETTh1   & 7   & 17{,}420  & 17{,}397  & 24 & 1 & \TCVETThOne \\
Solar   & 137 & 52{,}560  & 12{,}342  & 24 & 1 & \TCVSolar \\
\bottomrule
\end{tabular}
\end{table}
\begin{table}[H]
  \centering
  \caption{Mean MAE/RMSE for two dynamic-correlation energy datasets. We report MAE, RMSE and MAPE at horizons 6 and 12. Best, second-best and third-best results are highlighted.}
  \label{tab:energy}
  \resizebox{\textwidth}{!}{
    \begin{tabular}{l|ccc|ccc|ccc|ccc|c}
      \toprule
      & \multicolumn{6}{c|}{Germany (TCV = \TCVGermany)} & \multicolumn{6}{c|}{France (TCV = \TCVFrance)} & \\
      \cmidrule(lr){2-7}\cmidrule(lr){8-13}      
      Methods & \multicolumn{3}{c|}{H=6} & \multicolumn{3}{c|}{H=12} & \multicolumn{3}{c|}{H=6} & \multicolumn{3}{c}{H=12} & Rank \\
      \cmidrule(lr){2-4}\cmidrule(lr){5-7}\cmidrule(lr){8-10}\cmidrule(lr){11-13}
      & MAE & RMSE & MAPE & MAE & RMSE & MAPE & MAE & RMSE & MAPE & MAE & RMSE & MAPE & \\
      \midrule
      VAR           & 0.210 & 0.340 & 0.191 & 0.243 & 0.381 & 0.215 & 0.155 & 0.241 & 0.132 & 0.177 & 0.260 & 0.151 & 11 \\
      DLinear       & 0.230 & 0.320 & 0.201 & 0.266 & 0.368 & 0.231 & 0.170 & 0.230 & 0.140 & 0.196 & 0.259 & 0.158 & 12 \\
      LSTNet        & 0.165 & 0.290 & 0.145 & 0.193 & 0.346 & 0.172 & 0.150 & 0.220 & 0.129 & 0.177 & 0.263 & 0.146 & 9 \\
      Autoformer    & 0.177 & 0.325 & 0.157 & 0.204 & 0.376 & 0.181 & 0.143 & 0.230 & 0.122 & 0.165 & 0.263 & 0.141 & 8.5 \\
      Informer      & 0.248 & 0.299 & 0.211 & 0.283 & 0.324 & 0.239 & \cellcolor{RoyalBlue!5}0.119 & \cellcolor{RoyalBlue!5}0.194 & \cellcolor{RoyalBlue!5}0.102 & \cellcolor{RoyalBlue!5}0.137 & \cellcolor{RoyalBlue!5}0.217 & \cellcolor{RoyalBlue!5}0.116 & 5 \\
      Reformer      & 0.259 & 0.331 & 0.222 & 0.297 & 0.361 & 0.247 & 0.124 & 0.208 & 0.106 & 0.141 & 0.233 & 0.119 & 10 \\
      \midrule
      FourierGNN    & 0.097 & 0.166 & 0.086 & 0.110 & 0.186 & 0.098 & 0.084 & 0.143 & 0.073 & 0.096 & 0.164 & 0.083 & 3 \\
      StemGNN       & 0.155 & 0.259 & 0.137 & 0.179 & 0.285 & 0.158 & 0.128 & 0.189 & 0.109 & 0.148 & 0.206 & 0.124 & 7.5 \\
      TPGNN         & 0.086 & 0.152 & 0.079 & 0.099 & 0.173 & 0.090 & 0.077 & 0.139 & 0.070 & 0.089 & 0.158 & 0.081 & 2 \\
      MTGNN         & \cellcolor{RoyalBlue!5}0.013 & \cellcolor{RoyalBlue!5}0.031 & \cellcolor{RoyalBlue!5}0.030 & \cellcolor{RoyalBlue!5}0.016 & \cellcolor{RoyalBlue!5}0.034 & \cellcolor{RoyalBlue!5}0.038 & \cellcolor{RoyalBlue!5}0.010 & \cellcolor{RoyalBlue!5}0.020 & 0.070 & \cellcolor{RoyalBlue!5}0.012 & \cellcolor{RoyalBlue!5}0.023 & 0.093 & 2.5 \\
      GraphWaveNet  & \cellcolor{RoyalBlue!20}0.009 & \cellcolor{RoyalBlue!20}0.021 & \cellcolor{RoyalBlue!20}0.028 & \cellcolor{RoyalBlue!20}0.013 & \cellcolor{RoyalBlue!20}0.028 & \cellcolor{RoyalBlue!20}0.040 & \cellcolor{RoyalBlue!20}0.010 & \cellcolor{RoyalBlue!20}0.020 & \cellcolor{RoyalBlue!20}0.069 & \cellcolor{RoyalBlue!20}0.012 & \cellcolor{RoyalBlue!20}0.025 & \cellcolor{RoyalBlue!20}0.092 & 1.5 \\
      \midrule
      \proposed  & \cellcolor{RoyalBlue!40}0.008 & \cellcolor{RoyalBlue!40}0.019 & \cellcolor{RoyalBlue!40}0.023 & \cellcolor{RoyalBlue!40}0.010 & \cellcolor{RoyalBlue!40}0.022 & \cellcolor{RoyalBlue!40}0.030 & \cellcolor{RoyalBlue!40}0.008 & \cellcolor{RoyalBlue!40}0.016 & \cellcolor{RoyalBlue!40}0.060 & \cellcolor{RoyalBlue!40}0.010 & \cellcolor{RoyalBlue!40}0.020 & \cellcolor{RoyalBlue!40}0.079 & 0.5 \\
      \bottomrule
    \end{tabular}
  }
\end{table}
\section{Ablation Study}
\label{app:details of ablation study}
\paragraph{Ablation on Dynamic Graph Design} We evaluate the incremental contributions of our two design choices under the same training protocol. Recall that \textbf{D1} removes the second (adaptive) graph and keeps a single fixed/base graph; \textbf{D2} denotes dual-graph variants with targeted edits: (\emph{i}) \textbf{D2-Random} replaces the base graph with a random topology while retaining the adaptive branch, and (\emph{ii}) \textbf{D2-noGraph} disables the adaptive branch while keeping the base graph. Unless otherwise noted, depth and optimization are identical across variants. For tables reporting $H\in\{3,6,12\}$, the value at $H{=}K$ is the \emph{prefix-mean} over the first $K$ horizons.

\paragraph{Results on \textbf{Synthetic}.}
Across difficulty levels, the full model consistently yields the lowest MAE/RMSE at both short and long horizons. On the Easy split, our model is best at $H{=}1$ and maintains the lead through $H{=} 12$; as difficulty increases (Medium $\rightarrow$ Hard/Very Hard), the margin becomes more visible. These trends support that (i) keeping two graphs and (ii) mixing multi-hop information are complementary; removing either part (\textbf{D1} or \textbf{D2-Random}) hurts most at short horizons, while corrupting the base topology (\textbf{D2-Random}) mainly degrades late-horizon stability.
\paragraph{Results on \textbf{Exchange Rate}.}
We observe the same pattern. The full model (\textsc{\proposed}) outperforms \textbf{D1} and \textbf{D2} for $H{=}3,6,12$. \textbf{D2-noGraph} is notably worse at $H{=}3$, indicating the adaptive branch is crucial for near-term predictions. \textbf{D2-Random} remains competitive at short horizons (the adaptive branch compensates for an imperfect base graph) but lags behind at $H{=}12$. Overall, dual-graph aggregation with power mixing is the most robust across horizons.

\paragraph{Ablation on Dynamic Graph Construction.} We further evaluate the effect of different dynamic graph constructions under the same training protocol. Specifically, \textit{Data-Corr} constructs the dynamic graph from the raw-signal local correlation defined in Eq.~\eqref{eq:corr_graph}, while \textit{Data-Gradient} replaces the raw signal with first-order temporal differences to emphasize transient changes defined in Eq.~\eqref{eq:grad_graph}. We keep model depth, optimization settings, and all other components identical across variants. For each dataset, we report MAE/RMSE at horizons $H\in\{3,6,12\}$ and compute the degradation separately within each dataset. We report the result in \Cref{tab:ablation-corr-grad}.

\paragraph{Conclusion.}
The results show that the preferred dynamic graph construction method depends on the dataset's temporal characteristics. On \textsc{Exchange Rate}, the raw-signal data-correlation graph achieves slightly lower errors, especially for RMSE, suggesting that direct local correlations are sufficient when the underlying dynamics are relatively smooth and stable. In contrast, on \textsc{Syn-Medium} and \textsc{Syn-Very Hard}, the gradient-based graph consistently improves both MAE and RMSE across all prediction horizons, with average reductions of $2.1\%$/$4.3\%$ and $5.2\%$/$4.4\%$ in MAE/RMSE, respectively. This supports the motivation of using first-order temporal differences to isolate transient shocks and rapidly changing dependencies. Overall, the ablation confirms that gradient-based dynamic graphs are particularly beneficial in strongly time-varying regimes, while raw correlation graphs can remain competitive for smoother real-world signals.

\begin{table}[H]
\centering
\small
\renewcommand{\arraystretch}{1.15}
\caption{\textsc{Exchange Rate} forecasting results. Prefix-mean MAE/RMSE over prediction horizons $H$. Lower is better.}
\label{tab:exch-graph-prefix-mean}
\begin{tabular}{l|ccc|ccc}
\toprule
& \multicolumn{3}{c|}{MAE} & \multicolumn{3}{c}{RMSE} \\
Method & H=3 & H=6 & H=12 & H=3 & H=6 & H=12 \\
\midrule
D2-Random   & .009 & .009 & .010 & .009 & .012 & .015 \\
D2-NoGraph  & .009 & .011 & .011 & .009 & .017 & .019 \\
\textbf{D2 (Ours)} & \textbf{.005} & \textbf{.006} & \textbf{.007}
& \textbf{.008} & \textbf{.009} & \textbf{.012} \\
\bottomrule
\end{tabular}
\end{table}
\begin{table}[H]
\centering
\scriptsize
\renewcommand{\arraystretch}{1.2}
\caption{Ablation study of \proposed. Blue cells indicate the largest degradation relative to the full model (\textbf{Ours}). Removing either D1 or D2 consistently degrades performance, with D2 contributing more substantially overall}
\label{tab:ablation-merged}
\resizebox{\columnwidth}{!}{
\begin{tabular}{ll | *{3}{wc{25pt}} | wc{35pt} | *{3}{wc{25pt}} | wc{35pt} | *{3}{wc{25pt}} | wc{35pt}}
\toprule
& & \multicolumn{4}{c|}{Exchange Rate} & \multicolumn{4}{c|}{Syn-Medium} & \multicolumn{4}{c}{Syn-Very Hard} \\
Methods & Metrics & 3 & 6 & 12 & \textbf{Deg.} & 3 & 6 & 12 & \textbf{Deg.} & 3 & 6 & 12 & \textbf{Deg.} \\
\midrule
\multirow{2}{*}{Data-Gradient}
& MAE  & .006 & \cellcolor{RoyalBlue!20}\textbf{.006} & \cellcolor{RoyalBlue!20}\textbf{.007} & +6.7\% & \cellcolor{RoyalBlue!20}\textbf{.402} & \cellcolor{RoyalBlue!20}\textbf{.448} & \cellcolor{RoyalBlue!20}\textbf{.489} & -2.1\% & \cellcolor{RoyalBlue!20}\textbf{.439} & \cellcolor{RoyalBlue!20}\textbf{.455} & \cellcolor{RoyalBlue!20}\textbf{.465} & -5.2\% \\
& RMSE & .009 & \cellcolor{RoyalBlue!20}\textbf{.009} & .013 & +12.3\% & \cellcolor{RoyalBlue!20}\textbf{.529} & \cellcolor{RoyalBlue!20}\textbf{.552} & \cellcolor{RoyalBlue!20}\textbf{.562} & -4.3\% & \cellcolor{RoyalBlue!20}\textbf{.546} & \cellcolor{RoyalBlue!20}\textbf{.547} & \cellcolor{RoyalBlue!20}\textbf{.544} & -4.4\% \\
\midrule
\multirow{2}{*}{Data-Corr}
& MAE  & \cellcolor{RoyalBlue!20}\textbf{.005} & \cellcolor{RoyalBlue!20}\textbf{.006} & \cellcolor{RoyalBlue!20}\textbf{.007} & -- & .410 & .461 & .496 & -- & .462 & .480 & .491 & -- \\
& RMSE & \cellcolor{RoyalBlue!20}\textbf{.007} & \cellcolor{RoyalBlue!20}\textbf{.009} & \cellcolor{RoyalBlue!20}\textbf{.012} & -- & .544 & .572 & .602 & -- & .564 & .571 & .578 & -- \\
\bottomrule
\end{tabular}
}
\end{table}
\begin{table}[H]
\centering
\scriptsize
\renewcommand{\arraystretch}{1.2}
\caption{Ablation study on the construction of the dynamic graph in \proposed. Data-Corr is slightly better on Exchange Rate, whereas Data-Gradient consistently improves performance on Syn-Medium and Syn-Very Hard, indicating its advantage in capturing rapidly changing dynamics.}
\label{tab:ablation-corr-grad}
\resizebox{\columnwidth}{!}{
\begin{tabular}{ll | *{3}{wc{25pt}} | wc{35pt} | *{3}{wc{25pt}} | wc{35pt} | *{3}{wc{25pt}} | wc{35pt}}
\toprule
& & \multicolumn{4}{c|}{Exchange Rate} & \multicolumn{4}{c|}{Syn-Medium} & \multicolumn{4}{c}{Syn-Very Hard} \\
Methods & Metrics & 3 & 6 & 12 & \textbf{Deg.} & 3 & 6 & 12 & \textbf{Deg.} & 3 & 6 & 12 & \textbf{Deg.} \\
\midrule
\multirow{2}{*}{Data-Gradient}
& MAE  & .006 & \cellcolor{RoyalBlue!20}\textbf{.006} & \cellcolor{RoyalBlue!20}\textbf{.007} & +6.7\% & \cellcolor{RoyalBlue!20}\textbf{.402} & \cellcolor{RoyalBlue!20}\textbf{.448} & \cellcolor{RoyalBlue!20}\textbf{.489} & -2.1\% & \cellcolor{RoyalBlue!20}\textbf{.439} & \cellcolor{RoyalBlue!20}\textbf{.455} & \cellcolor{RoyalBlue!20}\textbf{.465} & -5.2\% \\
& RMSE & .009 & \cellcolor{RoyalBlue!20}\textbf{.009} & .013 & +12.3\% & \cellcolor{RoyalBlue!20}\textbf{.529} & \cellcolor{RoyalBlue!20}\textbf{.552} & \cellcolor{RoyalBlue!20}\textbf{.562} & -4.3\% & \cellcolor{RoyalBlue!20}\textbf{.546} & \cellcolor{RoyalBlue!20}\textbf{.547} & \cellcolor{RoyalBlue!20}\textbf{.544} & -4.4\% \\
\midrule
\multirow{2}{*}{Data-Corr}
& MAE  & \cellcolor{RoyalBlue!20}\textbf{.005} & \cellcolor{RoyalBlue!20}\textbf{.006} & \cellcolor{RoyalBlue!20}\textbf{.007} & -- & .410 & .461 & .496 & -- & .462 & .480 & .491 & -- \\
& RMSE & \cellcolor{RoyalBlue!20}\textbf{.007} & \cellcolor{RoyalBlue!20}\textbf{.009} & \cellcolor{RoyalBlue!20}\textbf{.012} & -- & .544 & .572 & .602 & -- & .564 & .571 & .578 & -- \\
\bottomrule
\end{tabular}
}
\end{table}
\end{document}